\documentclass[letterpaper]{article} 
\usepackage{aaai25}  
\usepackage{times}  
\usepackage{helvet}  
\usepackage{courier}  
\usepackage[hyphens]{url}  
\usepackage{graphicx} 
\usepackage{natbib}  
\usepackage{caption} 
\usepackage{subcaption}

\usepackage[linesnumbered,vlined,ruled,commentsnumbered,nofillcomment]{algorithm2e}
\usepackage{algorithm}
\usepackage{algorithmic}

\usepackage{newfloat}
\usepackage{listings}
\DeclareCaptionStyle{ruled}{labelfont=normalfont,labelsep=colon,strut=off} 
\floatstyle{ruled}
\newfloat{listing}{tb}{lst}{}
\floatname{listing}{Listing}
\usepackage{stmaryrd}
\usepackage{matharticle}
\usepackage{thmtools, thm-restate}

\usepackage{amsmath,amsthm,amssymb}

\newtheorem{theorem}{Theorem}

\newtheorem{lemma}{Lemma}

\newtheorem{definition}{Definition}

\newtheorem{theoremsub}{Theorem}[theorem]

\title{Hybrid Decentralized Optimization:\\
Leveraging Both First- and Zeroth-Order Optimizers for Faster Convergence}
\author{
    Matin Ansaripour\textsuperscript{\rm 1}\equalcontrib, Shayan Talaei\textsuperscript{\rm 2}\equalcontrib, Giorgi Nadiradze\textsuperscript{\rm 3}, Dan Alistarh\textsuperscript{\rm 3} 
}
\affiliations{
    \textsuperscript{\rm 1}École Polytechnique Fédérale de Lausanne (EPFL)\\
    \textsuperscript{\rm 2}Stanford University\\
    \textsuperscript{\rm 3}Institute of Science and Technology Austria (ISTA)\\
    stalaei@stanford.edu,
    matin.ansaripour@epfl.ch,
    \{giorgi.nadiradze, dan.alistarh\}@ista.ac.at

}

\begin{document}

\maketitle

\begin{abstract}
Distributed optimization is the standard way of speeding up machine learning training, and most of the research in the area focuses on distributed first-order, gradient-based methods. Yet, there are settings where some computationally-bounded nodes may not be able to implement \emph{first-order, gradient-based} optimization, while they could still contribute to joint optimization tasks. In this paper, we initiate the study of hybrid decentralized optimization, studying settings where nodes with zeroth-order and first-order optimization capabilities co-exist in a distributed system, and attempt to jointly solve an optimization task over some data distribution. We essentially show that, under reasonable parameter settings, such a system can not only withstand noisier zeroth-order agents but can even benefit from integrating such agents into the optimization process, rather than ignoring their information. At the core of our approach is a new analysis of distributed optimization with noisy and possibly-biased gradient estimators, which may be of independent interest. Our results hold for both convex and non-convex objectives. Experimental results on standard optimization tasks confirm our analysis, showing that hybrid first-zeroth order optimization can be practical, even when training deep neural networks.
\end{abstract}



\begin{links}
    \link{Code}{https://github.com/ShayanTalaei/HDO}
    \link{Extended version}{https://arxiv.org/abs/2210.07703}
\end{links}

\section{Introduction}

One key enabler of the extremely rapid recent progress of machine learning has been \emph{distributed optimization}: the ability to efficiently optimize over large quantities of data, and large parameter counts, among multiple nodes or devices, in order to share the computational load, and therefore reduce end-to-end training time. 
Distributed machine learning has become commonplace, and it is not unusual to encounter systems which distribute model training among tens or even hundreds of nodes. 

By and large, the standard distribution strategy in the context of machine learning tasks has been \emph{data-parallel}~\cite{bottou2010large}, using \emph{first-order} gradient estimators. 
 We can formalize this as follows: considering a classical empirical risk minimization setting, we have a set of samples $S$ from a distribution, and wish to minimize the function $f: \mathbb{R}^d 
\rightarrow \mathbb{R}$, which is the average of losses over samples from $S$. In other words, we wish to find $x^\star = \textnormal{ argmin }_{x} \sum_{s \in S} f_s (x) / |S|$. 
Assuming that we have $n$ compute nodes which can process samples in parallel, data-parallel SGD consists of iterations in which each node computes gradient estimator for a batch of samples, and then  nodes then exchange this information, either globally, via all-to-all communication, or \emph{pair-wise}. 
Specifically, in this paper we will focus on the highly-popular \emph{decentralized optimization case}, in which nodes interact in randomly chosen pairs, exchanging model information, following each local optimization step. 

There is already a vast amount of literature on decentralized optimization in the case where nodes have access to \emph{first-order, gradient-based} estimators. 
While this setting is prevalent, it does not cover the interesting case where, among the set of nodes, a fraction only have access to weaker, \emph{zeroth-order} gradient estimators, corresponding to less computationally-capable devices, but which may still possess useful local data and computation. 

In this paper, we initiate the study of \emph{hybrid decentralized optimization} in the latter setting. Specifically, we aim to answer the following key question:

\begin{center}
    \emph{Can zeroth-order estimators be integrated in a decentralized setting, and can they boost convergence?} 
\end{center}

Roughly, we show that the answer to this question is {affirmative}. 
To arrive at it, we must overcome a number of non-trivial technical obstacles, and the answer must be qualified by key parameters, such as the first-order/zeroth-order split in the population, and the estimator variance and bias. 
More precisely, a key difficulty we must overcome in the algorithm and in the analysis is the fact that, under standard implementations, zeroth-order estimators are \emph{biased}, breaking one of the key analytic assumptions in existing work on decentralized optimization, e.g.~\cite{lian2017can, wang2018cooperative, Koloskova*2020Decentralized, koloskova2020unified, nadiradze2021asynchronous}. 

Our analysis approach overcomes this obstacle and provides the first convergence bounds for hybrid decentralized optimization via a novel potential argument. 
Roughly, assuming a $d$-dimensional and $L$-smooth finite-sum objective function $f$, 
and a population of $n$ nodes, in which $n_1$ have first-order stochastic gradient estimators of variance $\sigma_1$, and $n_0$ have zeroth-order estimators of variance $\sigma_0$, then our analysis shows that the ``stochastic noise'' in the convergence of our hybrid decentralized optimization algorithm in this population is given, up to constants, by the following three quantities: 

\begin{equation}\label{eq:noise}
    \frac{\eta(d n_0\varsigma_0^2+n_1 \varsigma_1^2)}{n^2}, \frac{\eta(d n_0\sigma_0^2+n_1 \sigma_1^2)}{n^2}, \eta^2\big(\frac{ L d n_0}{n}\big)^k.
\end{equation}

\noindent In this expression, $\eta$ is the learning rate, and the quantities $\varsigma_1$ and $\varsigma_0$ are bounds on the \emph{average} variance of first-order and zeroth-order estimators at the nodes, respectively, given by the way in which the data is split among these two types of agents. 
Intuitively, the first term is the variance due to the (random) data split, whereas the second term is the added variance due to noise in the two types of gradient estimators. 
(The zeroth-order terms are scaled by the dimension, as is common in this case.)
The third term bounds \emph{the bias} induced by the zeroth-order gradient estimators, where $k$ equals 1 for the convex case and 2 for the non-convex case. 
Using this characterization, we show that there exist reasonable parameter settings such that, if zeroth-order nodes do not have extremely high variance, they may in fact be useful for convergence, especially since the third bias term can be controlled via the learning rate $\eta$.  

Our analysis approach should be of independent interest: first, we provide a simple and general way of characterizing convergence in a population mixing first- and zeroth-order agents, which can be easily parametrized given population and estimator properties, for both convex and non-convex objectives. 
(For instance, we can directly cover the case when the zeroth-order estimators are unbiased~\cite{Chen2020unbiased} as in this case the bias term becomes zero.) 
Second, we do so in a very general communication model which allows agents to interact at different rates (due to randomness), covering both the pair-wise interaction model~\cite{AADFP06, nadiradze2021asynchronous} and the global matching interactions model~\cite{lian2017can, wang2018cooperative, Koloskova*2020Decentralized, koloskova2020unified}.  

A key remaining question is whether the above characterization can be validated for practical setting. 
For this, we implemented our algorithm and examined the convergence under various optimization tasks, population relative sizes, and estimator implementations. 
Specifically, we implemented three different types of zeroth-order estimators: a standard biased one, e.g.~\cite{nesterov2017random}, a de-biased estimator~\cite{Chen2020unbiased}, and the novel gradient-free estimator of~\cite{baydin2022forwardmode}, and examined their behavior when mixed with first-order estimators. 
In brief, our results show that, even for high-dimensional and complex tasks, such as fine-tuning the ResNet-18~\cite{he2015deepresiduallearningimage}, or a Transformer model~\cite{vaswani2023attentionneed}, our approach continues to converge. Importantly, we observe that our approach allows a system to incorporate information from the zeroth-order agents in an efficient and robust, showing higher convergence speed relative to the case where only first-order information is considered for optimization.

\paragraph{Related Work.} 
The study of decentralized optimization algorithms dates back to~\citet{tsitsiklis1984problems}, and is related to the study of \emph{gossip} algorithms for information dissemination~\cite{kempe2003gossip, xiao2004fast}.  The distinguishing feature of this setting is that optimization occurs jointly, but in the absence of a coordinator node. 
Several classic first-order algorithms have been ported and analyzed in the gossip setting, such as subgradient methods for convex objectives~\cite{nedic2009distributed, johansson2009randomized, shamir2014distributed} or ADMM~\cite{wei2012distributed, iutzeler2013asynchronous}. References~\cite{lian2017can, lian2017asynchronous, assran2018stochastic} consider SGD-type algorithms in the non-convex setting, while references~\cite{tang2018decentralization, Koloskova*2020Decentralized, nadiradze2021asynchronous} analyzed the use of quantization in the gossip setting.
By contrast, zeroth-order optimization has been relatively less investigated: \citet{sahu2020decentralized} proposes a distributed deterministic zeroth-order Frank-Wolfe-type algorithm, whereas other works by~\cite{yuan2021distributed} and \cite{pmlr-v202-mhanna23a} investigated the rates which can be achieved by decentralized zeroth-order algorithms, proposing multi-stage methods which can match the rate of centralized algorithms in some parameter regimes. Relative to the latter reference, we focus on simpler decentralized algorithms, which can easily interface with first-order optimizers, and perform a significantly more in-depth experimental validation. 

Stochastic zeroth-order optimization has been classically applied for gradient-free optimization of convex functions, e.g.~\cite{nesterov2017random}, 
and has been extended to tackling high-dimensionality and saddle-point constraints, e.g.~\cite{balasubramanian2019zeroth}. 
(The area has tight connections to bandit online optimization, under time-varying objective functions, e.g.~\cite{flaxman2004online, agarwal2010optimal, shamir2017optimal}; however, our results are not immediately relevant to this direction, as we are interested in interactions with agents possessing first-order information as well.) 
In this paper, we also investigate improved single-point function evaluation for better gradient estimation~\cite{kuhn2022zerothimaginary} as well as the forward-mode unbiased estimator of~\citet{baydin2022forwardmode}. 

\section{Preliminaries}

\subsection{The System Model}

We consider a standard model for the decentralized optimization setting, which is similar to~\cite{Koloskova*2020Decentralized, koloskova2020unified, lian2017can, nadiradze2021asynchronous}.
Specifically, we have $n \geq 2$ agents, of which $n_0$ agents have zeroth-order gradient oracles, and $n_1$ have first-order gradient oracles. (We describe the exact optimization setup in the next section.) 
Beyond their oracle type, the agents are assumed to be anonymous for the purposes of the protocol. 
The execution will proceed in discrete \emph{steps}, or \emph{rounds}, where in each step, two agents are chosen to interact, uniformly at random. Specifically, when chosen, each agent performs some local computation, e.g. obtains some gradient information from their local oracle. 
Then, the two agents exchange parameter information, and update their local models, after which they are ready to proceed to the next round. Notice that this random interaction model is asynchronous, in the sense that the number of interactions taken by agents up to some point in time may be different, due to randomness. 
The basic unit of time used in the analysis, which we call \emph{fine-grained time}, will be the total number of interactions among agents up to some given point in the execution. To express global progress, we will consider \emph{parallel time}, which is the \emph{average} number of interactions up to some point, and can be obtained by dividing by $n$ the total number of interactions. This corresponds to the intuition that $\Theta(n)$ interactions may occur in parallel. 
In experiments, we will examine the convergence of the local model at a fixed node. 

This model is an instantiation of the classic \emph{population model} of distributed computing~\cite{AADFP06}, in an optimization setting. The model is similar to the one adopted by~\citet{nadiradze2021asynchronous} for analyzing asynchronous decentralized SGD, 
and is more general than the ones adopted by~\citet{Koloskova*2020Decentralized, koloskova2020unified, lian2017can, wang2018cooperative} for decentralized analysis, since the latter assume that nodes are paired via perfect global random matchings in each round. 
(Our analysis would easily extend to global matching, yielding virtually the same results.)


\subsection{Optimization Setup}

We assume each node $i$ has a local data distribution $\mathcal{D}^i$, and that the loss function corresponding to the samples at node $i$, denoted by $f^i(x):\mathbb{R}^d \shortrightarrow \mathbb{R}$ can be approximated using its stochastic form $F^i(x, \xi^i)$ for each parameter $x \in \mathbb{R}^d$ and (randomly chosen) sample $\xi^i \sim \mathcal{D}^i$, where $f^i(x) = \E_{\xi^i \sim \mathcal{D}^i}\big[ F^i(x, \xi^i) \big]$. For simplicity of notation, we assume that nodes in the set $N_0=\{1,2, ..., n_0\}$ are zeroth-order nodes and the nodes in the set $N_1 = [n]/N_0$ are first-order nodes. Let $n_0$ and $n_1$ be the sizes of the sets $N_0$ and $N_1$ correspondingly.

In this setup nodes communicate to solve a distributed stochastic optimization problem, i.e. $$f^* = \underset{x \in \mathbb{R}^d}{min} \Big[ f(x) := \frac{1}{n_0} \sum_{i \in N_0} f^i(x)  + \frac{1}{n_1} \sum_{i \in N_1}f^i(x)\Big].$$

This means that we wish to optimize the function $f$ which corresponds to the loss over all data samples.
Since in the analysis we will wish to throttle the ratio of zeroth-order to first-order agents, we split the entire data among zeroth-order nodes, and we do the same thing for the first-order nodes. 
(Our analysis can be extended to settings where this is not the case, but this will allow us for instance to study what happens when either $n_0$ or $n_1$ goes to zero, without changing our objective function.)
We make the following assumptions on the optimization objectives: 

\begin{restatable}[Strong convexity]{assumption}{stronglyConvex} \label{asmp:strongly_convex} We assume that the function $f$ is strongly convex with parameter $\ell>0$, i.e. for all $x, y \in \mathbb{R}^d$:
\begin{equation*}
(x - y)^T(\nabla f(x) - \nabla f(y)) \geq \ell \| x - y\|^2.
\end{equation*}
\end{restatable}

\begin{restatable}[Smooth gradient]{assumption}{smoothGradient}\label{asmp:lipschitz} All the stochastic gradients $\nabla F^i$ are L-Lipschitz for some constant $L > 0$, i.e. for all $\xi^i \sim \mathcal{D}^i$ and $x, y \in \mathbb{R}^d$:
\begin{equation}
    \| \nabla F^i(x, \xi^i) - \nabla F^i(y, \xi^i) \| \leq L \|x - y\|.
\end{equation}
If in addition $F^i$ are convex functions, then
\begin{align*}
    \| \nabla F^i(x, \xi^i) - \nabla F^i(y, \xi^i) \|& \leq\\ 2L (F^i(x, \xi^i)) - F^i(y,& \xi^i) - \langle x-y, \nabla F^i(y, \xi^i) \rangle).
\end{align*}
\end{restatable}

Using Assumption \ref{asmp:lipschitz}, one can easily find that the gradients of  $f$ and $f^i(x)$ $\forall i \in [n]$ are also satisfying the above inequalities.
Further, we make the following assumptions about the data split and the stochastic gradient estimators: 

\begin{restatable}[Balanced data distribution]{assumption}{globalVariance}\label{asmp:global_variance} The average variance of $\nabla f^i(x)$s for both zero and first order nodes is bounded by a global constant values, i.e. for all $x \in \mathbb{R}^d$:
\begin{align*}
\frac{1}{n_0}\underset{i \in N_0}{\sum} &\| \nabla f^i(x)  - \nabla f(x) \|^2 \leq \varsigma_0^2; \\
\frac{1}{n_1}\underset{i \in N_1}{\sum} &\| \nabla f^i(x)  - \nabla f(x) \|^2 \leq \varsigma_1^2.
\end{align*}
\end{restatable}

\begin{restatable}[Unbiasedness and bounded variance]{assumption}{unbiasedness}\label{asmp:unbiasedness_bounded_local_variance_of_F} For each i, $\nabla F^i(x, \xi^i)$ is an unbiased estimator of $\nabla f^i(x)$ and its variance is bounded by a constant $s_i^2$, i.e. for all $x \in \mathbb{R}^d$:
\begin{eqnarray*}
\E_{\xi^i \sim \mathcal{D}^i}[\nabla F^i(x, \xi^i)] = \nabla f^i(x); \\ \E_{\xi^i}\| \nabla F^i(x, \xi^i) - \nabla f^i(x)\| \leq s_i^2.
\end{eqnarray*}
\end{restatable} 

Each node has access to an estimator $G^i(x)$ that estimates the local gradient $\nabla f^i(x)$ at point $x$. 
For nodes which can perform the gradient computation over a batch of data, i.e. first-order nodes, $G^i(x)$ is $\nabla F^i(x, \xi^i)$, where $\xi^i \sim \mathcal{D}^i$.

\begin{definition}
    We define the average of $s_i$ for the zeroth and first order populations as $\sigma_0^2$ and $\sigma_1^2$ respectively. Formally, we define
    \begin{align*}
        \sigma_0^2 := \frac{1}{n_0} \sum_{i \in n_0} s_i^2,\quad \sigma_1^2 := \frac{1}{n_1} \sum_{i \in n_1} s_i^2.
    \end{align*}
\end{definition}

\subsection{Zeroth-order Optimization}

We now provide a brief introduction relative to standard basic facts and assumptions concerning zeroth-order optimization. 
Let the function $f^i_\nu(x):=\E_u[f^i(x+\nu u)],$ $u\sim N(0, I_d)$ be the smoothed version of each function $f^i(x)$. Then, node $i$ can estimate the gradient of $f^i_\nu$ by only evaluating some points of $f^i$. 
\begin{restatable}[Zeroth-order estimator]{definition}{zerothEstimator}
\label{def: zeroth-order estimator}
\begin{equation}
    G^i_\nu(x, u, \xi^i)=\frac{F^i(x+\nu u, \xi^i)-F^i(x, \xi^i)}{\nu}u,
\end{equation}
where $u \sim N(0, I_d)$ and $\xi^i \sim \mathcal{D}^i$.\
\end{restatable}
Note that under Assumption \ref{asmp:unbiasedness_bounded_local_variance_of_F}, one can easily prove that $G^i_\nu(x, u, \xi^i)$ is an unbiased estimator of $\nabla f^i_\nu$ since
\begin{align} \label{E(G_v)}
    \E_{u, \xi^i} [ G^i_\nu(x, u, \xi^i) ] &= \E_u [ \frac{f^i(x+\nu u) - f^i(x)}{\nu}u ] \nonumber \\&= \nabla f^i_\nu(x).
\end{align}

As a technical note, in our analysis we will set $\nu:=\frac{\eta}{c}$, where $\eta$ is the learning rate and $c$ is a constant to be defined later. Therefore, for simplicity we can define $G^i(x):=G^i_\nu(x, u, \xi^i)$, where $G^i_\nu(x, u, \xi)$ is as defined in Definition~\ref{def: zeroth-order estimator} and $\nu=\frac{\eta}{c}$. Since zeroth-order nodes cannot perform gradient computation directly, we  use this $G^i(x)$ as their gradient estimator.
We restate the following well-known fact:
\begin{restatable}[\cite{nesterov2017random}, Theorem 1.1 in~\cite{balasubramanian2019zeroth}]{lemma}{nestrov} 
\label{smth_approx} For a Gaussian random vector $u\sim N(0,I_d) $ we have that 
\begin{align}\label{eq:l2gauss}
\E[\|u\|^k] \le (d+k)^{k/2}
\end{align} 
for any $k \ge 2$. Moreover, the following statements hold for any function $f$ whose gradient is Lipschitz continuous with constant $L$.
\begin{itemize}
\item [a)]The gradient of $f_{\nu}$ is Lipschitz continuous with constant $L_{\nu}$ such that $L_{\nu} \le L$.
\item [b)] For any $x \in \mathbb{R}^d$,
\begin{align*}
|f_{\nu}(x)-f(x)| &\le \frac{\nu^2}{2} L d,\\ 
\|\nabla f_{\nu}(x) - \nabla f^i(x)\| &\le \frac{\nu}{2}L (d+3)^{\frac{3}{2}}.
\end{align*}
\item [c)]
For any $x \in \mathbb{R} ^n$,
\begin{eqnarray*} \label{stoch_smth_approx_grad}
\frac{1}{\nu^2}\E_u[\{f(x+\nu u)-f(x)\}^2\|u\|^2] \le \\ \frac{ \nu^2}{2}L^2(d+6)^3 + 2(d+4)\|\nabla f(x)\|^2.
\end{eqnarray*}
\end{itemize}
\label{rand_smth_close_grad}
\end{restatable}

\section{The HDO Algorithm}

\paragraph{Algorithm Description.} 
We now describe a decentralized optimization algorithm, designed to be executed by a population of $n$ nodes, interacting in pairs chosen uniformly at random as per our model. We assume that $n_1$ of the nodes have access to first-order estimators and $n_0$ of them have access to zeroth-order estimators, hence $n=n_1+n_0$. Two copies of the training data are distributed, once among the first-orders and once among the zeroth-orders. Thus, each first- and zeroth-order node has access to $\frac{1}{n_1}$, $\frac{1}{n_0}$ of the entire training data, respectively.  
We assume that each node $i$ has access to a local stochastic estimator of the gradient, which we denote by $G^i$, and maintains a model estimate $X^i$, as well as the global learning rate $\eta$. Without loss of generality, we assume that the models are initialized to the same randomly-chosen point. 
Specifically, upon every interaction, the interacting agents $i$ and $j$ perform the following steps:

\begin{algorithm}[h]
	\caption{HDO pseudocode for each interaction between randomy chosen nodes $i$ and $j$} \label{algo:hdo}
	~\tcp*[h]{Nodes perform local steps.}\\
        $X^i \gets X^i - \eta G^i(X^i)$; \\
        $X^j \gets X^j - \eta G^j(X^j)$; \\
        ~\tcp*[h]{Nodes average their local models.}\\
        ${avg} \gets (X^i + X^j) / 2$; \\
        $X^i \gets avg$; \\
        $X^j \gets avg$; \\
\end{algorithm}

In a nutshell, upon each interaction, each node first performs a local model update based on its estimator, and then nodes average their local models following the interaction. We do not distinguish between estimator types in this interaction. The nodes are then ready to proceed to the next round. 



\section{The Convergence of the HDO Algorithm}

This section is dedicated to proving that the following result

\begin{restatable}{theorem}{maintheorem} \label{thm:main}
Assume an objective function $f : \mathbb{R}^d \shortrightarrow \mathbb{R}$, equal to the average loss over all data samples, whose optimum $x^*$ we are trying to find using Algorithm~\ref{algo:hdo}. Let $n_0$ be the number of zeroth-order nodes, and $n_1$ be the number of first-order agents. Given the data split described in the previous section, let $f_i$ be the local objective function of node $i$. Assume that zeroth-order nodes use estimators with $\nu=\frac{\eta}{\sqrt{d}}$. Let the total number of steps in the algorithm $T$ and $\mu_t = \sum_{i = 1}^n X^i_t/n$, then we can derive the following convergence rates.

\paragraph{Non-Convex:} Under assumptions \ref{asmp:lipschitz}, \ref{asmp:global_variance} and \ref{asmp:unbiasedness_bounded_local_variance_of_F}, and letting $T$ be large enough such that $T = \Omega\left(\max\big\{\frac{L^2(dn_0 + n_1)^2}{dn^2}, n^2L^2, nn_0^3d\big\}\right)$, we have
\begin{align*}
    &\frac{1}{T}\sum_{t=0}^{T-1}\E\|\nabla f(\mu_t)\|^2 = \sqrt{\frac{d}{T}} \times O \Bigg( \big(f(\mu_0) - f^*\big)+ \\& L \big( \frac{dn_0\varsigma_0^2+n_1 \varsigma_1^2}{nd} \big)
    + L \big( \frac{d n_0\sigma_0^2+n_1\sigma_1^2}{nd} \big)
    + L^2\sqrt{\frac{n_0}{n}}\Bigg).
\end{align*}
Adding, convexity (Assumption \ref{asmp:strongly_convex}), we prove the following:
\paragraph{Strongly Convex:} If we assume that the functions $f$ and $f_i$ satisfy Assumptions \ref{asmp:strongly_convex}, \ref{asmp:lipschitz}, \ref{asmp:global_variance} and \ref{asmp:unbiasedness_bounded_local_variance_of_F}, and let $T$ be large enough such that $\frac{T}{\log T } = \Omega\left(\frac{n(d+n)(L+1)\left(\frac{1}{\ell}+1\right)}{\ell}\right)$, and let the learning rate be $\eta = \frac{4n\log T }{T \ell}$. For $1 \le t \le T$, let the sequence of weights $w_t$ be given by $w_t = \left(1-\frac{\eta \ell}{2n}\right)^{-t}$ and let $S_T = \sum_{t=1}^{T} w_T$. Finally, define $y_T=\sum_{t=1}^T \frac{w_t \mu_{t-1}}{S_T}$ to be the mean over local model parameters. Then, we can show that HDO provides the following convergence rate:
\begin{align*}
\E[&f(y_T) - f(x^*)]+\frac{\ell \E\|\mu_{T}-x^*\|^2}{8} \\&=
O\Bigg(\frac{L \|\mu_0-x^*\|^2}{T\log T } +\frac{\log (T) (d n_0\varsigma_0^2+n_1 \varsigma_1^2)}{T \ell n} \\&\quad\quad\quad+
\frac{\log(T)(d n_0\sigma_0^2+n_1 \sigma_1^2)}{T \ell n} +\frac{\log(T) d n_0}{T \ell n}\Bigg).
\end{align*}

\end{restatable}

\paragraph{Speedup.} 
Here, the time $T$ refers to the \emph{total number of interactions} among agents, as opposed to \emph{parallel time}, corresponding to the average number of interactions $T / n$. 
These rates are reminiscent of sequential SGD. However, there are some distinctions: we are counting the total number of gradient oracle queries by the nodes, and there are some additional trailing terms, whose meanings we discuss below. 

We interpret this formula from the perspective of an arbitrary local model. For this, notice that the notion of \emph{parallel time} corresponding to the number of total interactions $T$, which is by definition $T_p = T / n$,  corresponds (up to constants) to the \emph{average} number of interactions and gradient oracle queries performed by each node up to time $T$. 
Therefore, in the strongly convex setup, for any single model, convergence with respect to its number of performed SGD steps $T_p$ would be $O( \log(nT_p) / (n T_p))$ (assuming all parameters are constant), which would correspond to $\Omega(\frac{n}{\log(nT_p)})=
\Omega(\frac{n}{\log(T)})$ speedup compared to a variant of sequential SGD. 
Notice that this is quite favorable to our algorithm, since we are considering  \emph{biased} zeroth-order estimators for some of the nodes in the population. Hence, assuming that $T$ is polynomial in $n$, we get an almost-linear speedup of $\Omega\Big(\frac{n}{\log(n)}\Big)$. Similarly, in the non-convex case, we get a speedup $\Omega(\sqrt{n})$, which shows the scalability of our algorithm.

\paragraph{Impact of Zeroth-Order Nodes.}
Notice that our convergence bounds cleanly separate in the terms which come from zeroth-order nodes and terms which come from first-order nodes. For $n_0=0$, we get asymptotically the same bound as we would get if all nodes performed pure first-order SGD steps. Similarly, when $n_0=n$ we should be able to achieve asymptotically-optimal convergence for biased zeroth-order estimators. Further, notice that, if the bias is negligible, then the last term in each of the upper bounds disappears, and we obtain a trade-off between two populations with different variances. We can also observe the following theoretical threshold: we asymptotically match the convergence rate in the case with all nodes performing SGD steps, as long as $d n_0 = O(n)$ (assuming all other parameters are constant).

\subsection{Analysis}
As an example, we discuss the proof overview for the strongly convex case. The notations and proof steps are closely aligned with those used in the non-convex case.
\paragraph{Proof Overview.} 
The convergence proof, given in full in the Appendix, can be split conceptually into two steps. 
The first aims to bound the variance of the local models $X^i_t$ for each time step $t$ and node $i$ with respect to the mean $\mu_t = \sum_i X^i_t$. It views this variance as a potential $\Gamma_t$, which as we show has supermartingale-like behavior for small enough learning rate: specifically, this quantity tends to increase due to gradient steps, but is pushed towards the mean $\mu_t$ by the averaging process. 

The key component here is Lemma~\ref{lem:GammaBoundPerStepHelper}, which carefully bounds the evolution of the potential at a step, by modeling optimization as a dynamic load balancing process: each interaction corresponds to a \emph{weight generation} step (in which gradient estimators are generated) and a \emph{load balancing step}, in which the ``loads'' of the two nodes (corresponding to their model values) are balanced through averaging. 

In the second step, we first bound the rate at which the mean $\mu_t$ converges towards $x^*$, where we crucially (and carefully) leverage the variance bound obtained above. The main challenge in this part is dealing with biased zeroth-order estimators. In fact, even dealing with biased first-order estimators is not trivial, since for example, they are the main reason for the usage of error feedback when stochastic gradients are compressed using biased quantization~\cite{TopK}.
This is our second key technical result. 

With this in hand, we can complete the proof by applying a standard argument which characterizes the rate at which $\E[f(y_T)-f(x^*)]$ and $\E[\|\mu_t-x^*\|^2$ converge towards $0$.

\paragraph{Notation and Preliminaries.} In this section, we provide a more in-depth sketch of the analysis of the HDO protocol. 
We begin with some notation. 
Recall that $n$ is the number of nodes, split into first-order ($n_1$) and zeroth-order ($n_0$). 
We will analyze a sequence of \emph{time steps} $t = 1, 2,\ldots, T$, each corresponding to an individual interaction between two nodes, which are usually denoted by $i$ and $j$. 

\textbf{Step 1: Parameter Concentration.} Next, let $X_t$ be a vector of model estimates at time step $t$, that is $X_t=(X_t^1, X_t^2, ..., X_t^n)$.
Also, let $\mu_t=\frac{1}{n} \sum\limits_{i=1}^n X_t^i$, be an average estimate at time step $t$. The following potential function measures the variance of the models: 
\begin{equation*}
\Gamma_t=\frac{1}{n} \sum_{i=1}^n \|X_t^i-\mu_t \|^2.
\end{equation*}


With this in place, one of our key technical results is to provide a supermartingale-type bound on the evolution of the potential $\Gamma_t$, in terms $\eta$, and average second moment of estimators at step $t$, defined as 
$M_t^G := \frac{1}{n}\sum_i \big\|G^i(X_t^i)\big\|^2$.

\begin{restatable} {lemma}{GammaBoundPerStepHelper} 
 \label{lem:GammaBoundPerStepHelper}
For any time step $t$ :
\begin{align*}
\E\big[ \Gamma_{t+1} \big] \leq \big( 1 - \frac{1}{2n}\big)\E\big[ \Gamma_t \big] + \frac{4}{n}\eta^2\E\big[M_t^G\big].
\end{align*}
\end{restatable}

Notice that, if we had a universal second moment bound on the estimators,
that is, for any vector $X$ and node $i$ $\E\big\|G^i(X)\big\|^2 \le M$, for some $M > 0$, then we would be able to unroll the recursion, and,  
for any $t \ge 0$ upper bound $E[\Gamma_t]$ by $\eta^2 M^2$.
In the absence of such upper bound we must derive the following upper bound on $\E\big[M_t^G\big]$:

\begin{restatable}{lemma}{MGBound} \label{lem:h M^G}
Assume $\nu: =\frac{\eta}{c}$ is fixed, where $\eta$ and $c$ are the learning rate and a constant respectively. Then, for any time step t we have:  
\begin{align*}
\E\big[ M_t^G\big] &\leq 
6(d+4)L^2 \E[\Gamma_t] + \frac{6(d+4)n_0\varsigma_0^2+3n_1 \varsigma_1^2}{n}\\&+6(2d+9)L\E[f(\mu_t)-f(x^*)] \\&+ \frac{2(d+4) \sum_{i \in N_0} s_i^2+\sum_{i \in N_1} s_i^2}{n} \\&+ \eta^2 \frac{n_0}{2nc^2}L^2(d+6)^3.
\end{align*}
\end{restatable}

First,  we check how this upper bound affects the upper bound given by Lemma \ref{lem:GammaBoundPerStepHelper}.
For small enough $\eta$, the term containing $\E[\Gamma_t]$ (which comes from the upper bound on $\E\big[ M_t^G\big]$) can be upper bounded by $\frac{1}{4n} \E[\Gamma_t]$, and hence it will just change the factor in front of 
$\E[\Gamma_t]$ to $(1-1/4n)$.

Second, since the above bound contains the term with $\E[f(\mu_t)-f(x^*)]$
we are not able to bound the potential $\Gamma$ per step, instead, 
for weights $w_t = (1-\frac{\eta \ell}{2n})^{-t}$, we can upper bound 
$\sum_{t=1}^{T} w_t \E[\Gamma_{t-1}]$ (please see Lemma \ref{lem:h sum_gamma} in the Appendix). The crucial property is that  the upper bound on the weighted sum $\sum_{t=1}^{T} w_t \E[\Gamma_{t-1}]$, is  
$$O(\eta^2 \sum_{t=1}^{T} w_t \E[f(\mu_{t-1})-f(x^*)])+\sum_{t=1}^{T} w_t O(\eta^2).$$
(for simplicity, above we assumed that all other parameters are constant.)

\paragraph{Step 2: Convergence of the Mean and Risk Bound.}  
The above result allows us to characterize how well the individual parameters are concentrated around their mean. In turn, this will allow us to provide a recurrence for how fast the parameter average is moving towards the optimum.
To help with the intuition, we provide the lemma which is simplified  version  of the one given in the additional material:

\begin{lemma} \label{lem:supmartingale_v2}
For small enough $\eta$  and $t \ge 1$ we have that:
\begin{align*}
\E \Big\| \mu_{t}  -x^* \Big \|^2 &\le (1-\frac{\ell\eta}{2n})\E\|\mu_{t-1}-x^*\|^2 \\&- \Omega(\frac{\eta}{n}) \E\big[f(\mu_{t-1})-f(x^*)\big]
\\&+ O\left( \frac{\eta}{n}\right)\E[\Gamma_{t-1}]+O(\frac{\eta^2}{n^2}).
\end{align*}
\end{lemma}


Note that $O$ and $\Omega$ hide all other parameters (we assume that all other parameters are constant).
As mentioned,  \emph{the main challenge in the proof of this lemma
is taking care of biased zeroth-order estimators}. 

Recall that $w_t = (1-\frac{\eta \ell}{2n})^{-t}$, by definition. 
We proceed by multiplying both sides of the above inequality by $w_{t}$
and then summing it up for $1 \le t \le T$. 
Then, once we plug the upper bound on $\sum_{t=1}^{T} w_t \E[\Gamma_{t-1}]$, for small enough $\eta$ the term  $O(\frac{\eta}{n})O(\eta^2 \sum_{t=1}^{T} w_t \E[f(\mu_{t-1}-f(x^*)])$ vanishes as it is dominated by the term $-\sum_{t=1}^T \Omega(\frac{\eta}{n}) \E\big[f(\mu_{t-1})-f(x^*)\big]$.

We get the final convergence bound after some simple calculations involving division of both sides by $S_T=\sum_{t=1}^T w_t$, and using $\eta=\frac{4n log(T)}{T \ell}$  together with the upper bound on $T$
(in turn, this makes sure that $\eta$ is small enough, so that all upper bounds we mentioned hold).

\section{Experimental Results}

\paragraph{Experimental Setup and Goals.}
In this section, we validate our results by simulating the HDO algorithm under different conditions, including varying the number of nodes, the ratios between first-order (FO) and zeroth-order (ZO) nodes, and the "strength" of the zeroth-order gradient estimators. Our focus is on the algorithm's convergence behavior, relative to the total number of optimization steps, which we measure by tracking the loss over time or the accuracy on the hidden validation set.
Our goal is to determine whether a hybrid system, combining both FO and ZO agents, can achieve better convergence compared to a system that relies on a single type of agent. Additionally, we aim to demonstrate the convergence speed, in terms of training loss at a fixed node, for different sizes of \emph{mono-type populations}, each consisting of only one type of estimator.

In the implementation of HDO, each step involves nodes updating their models, followed by the formation of $O(n)$ random disjoint pairs. Each pair exchanges their models and replaces their own with the averaged model. To demonstrate the effectiveness of zeroth-order nodes under reasonable parameter settings, we evaluate the mean validation loss and accuracy across all nodes and analyze the consensus of the models by measuring the standard deviation of losses. Additional details on the models, datasets, and the complete experimental setup are provided in the Appendix.


\begin{figure}[t]
\centering
\includegraphics[width=1.0\columnwidth]{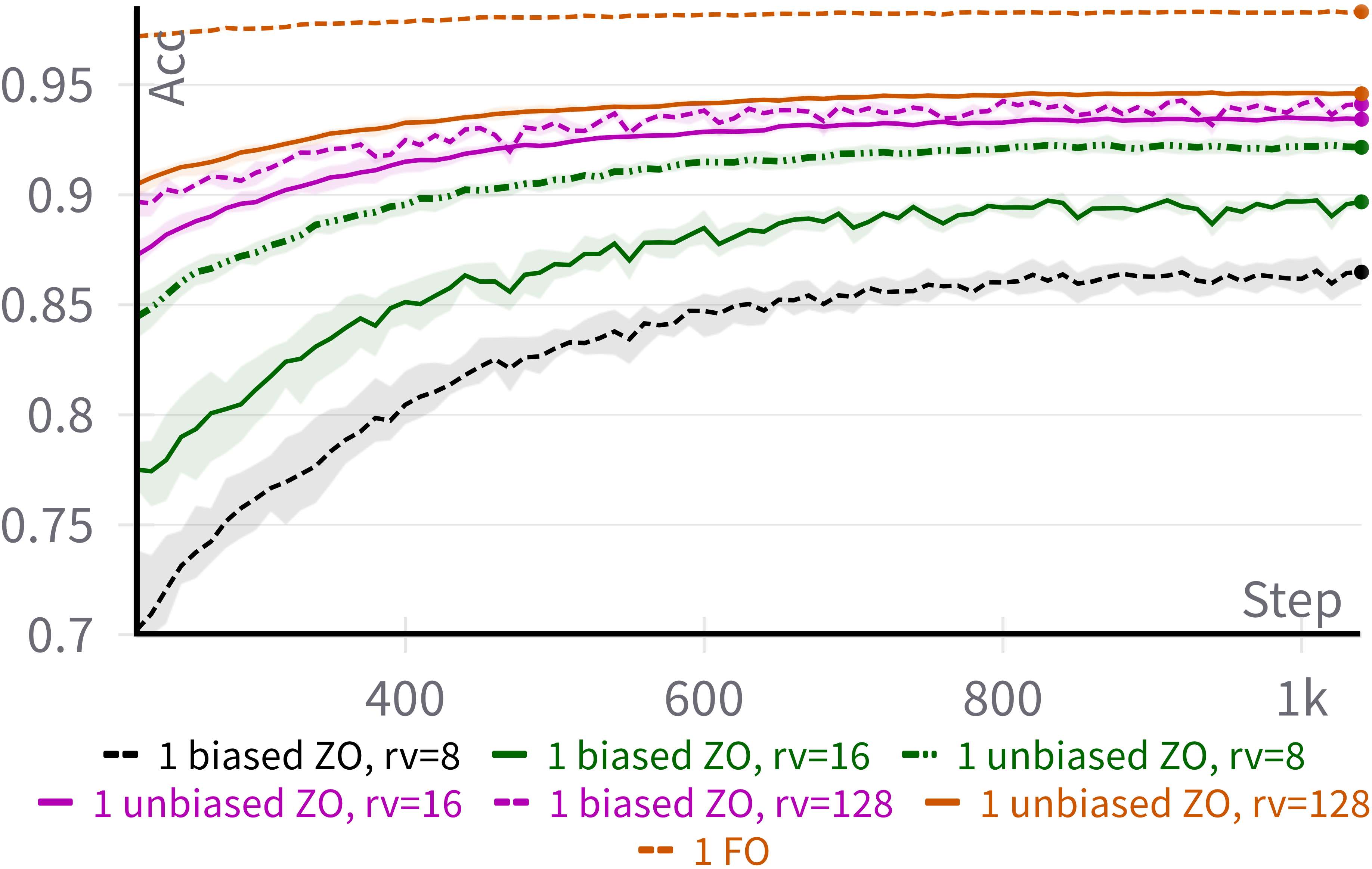} 
\caption{Number of random vectors (rv) impact on the biased/unbiased ZO estimators Accuracy (Acc), using a CNN model on MNIST.}
\label{fig:cnn_acc}
\end{figure}

\begin{figure}[t]
\centering
\includegraphics[width=1.0\columnwidth]{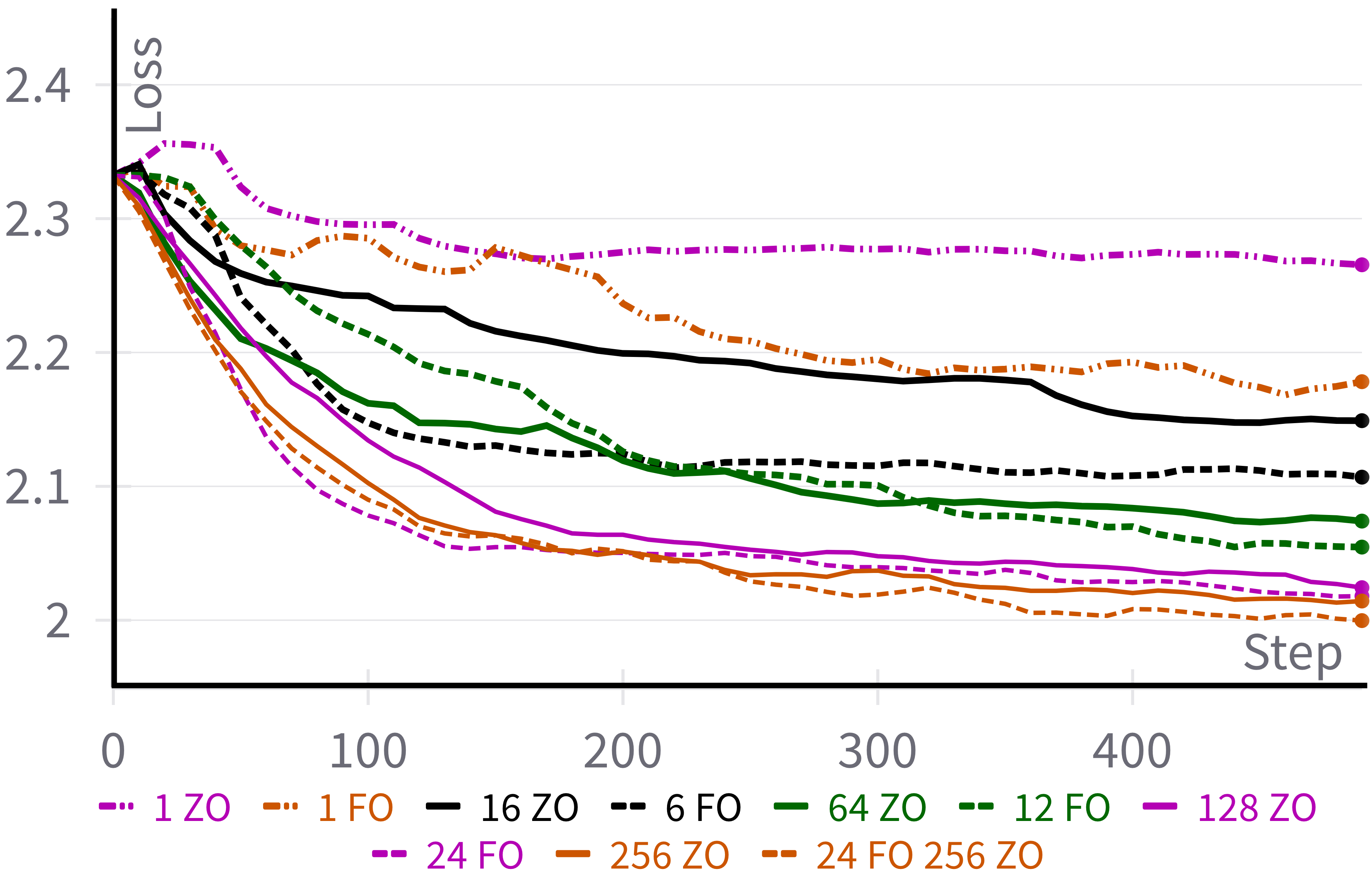} 

\caption{Validation loss vs. various population configurations for regression model on MNIST.}
\label{fig:logistic_loss}
\end{figure}

\begin{figure}[t]
\centering
\includegraphics[width=1.0\columnwidth]{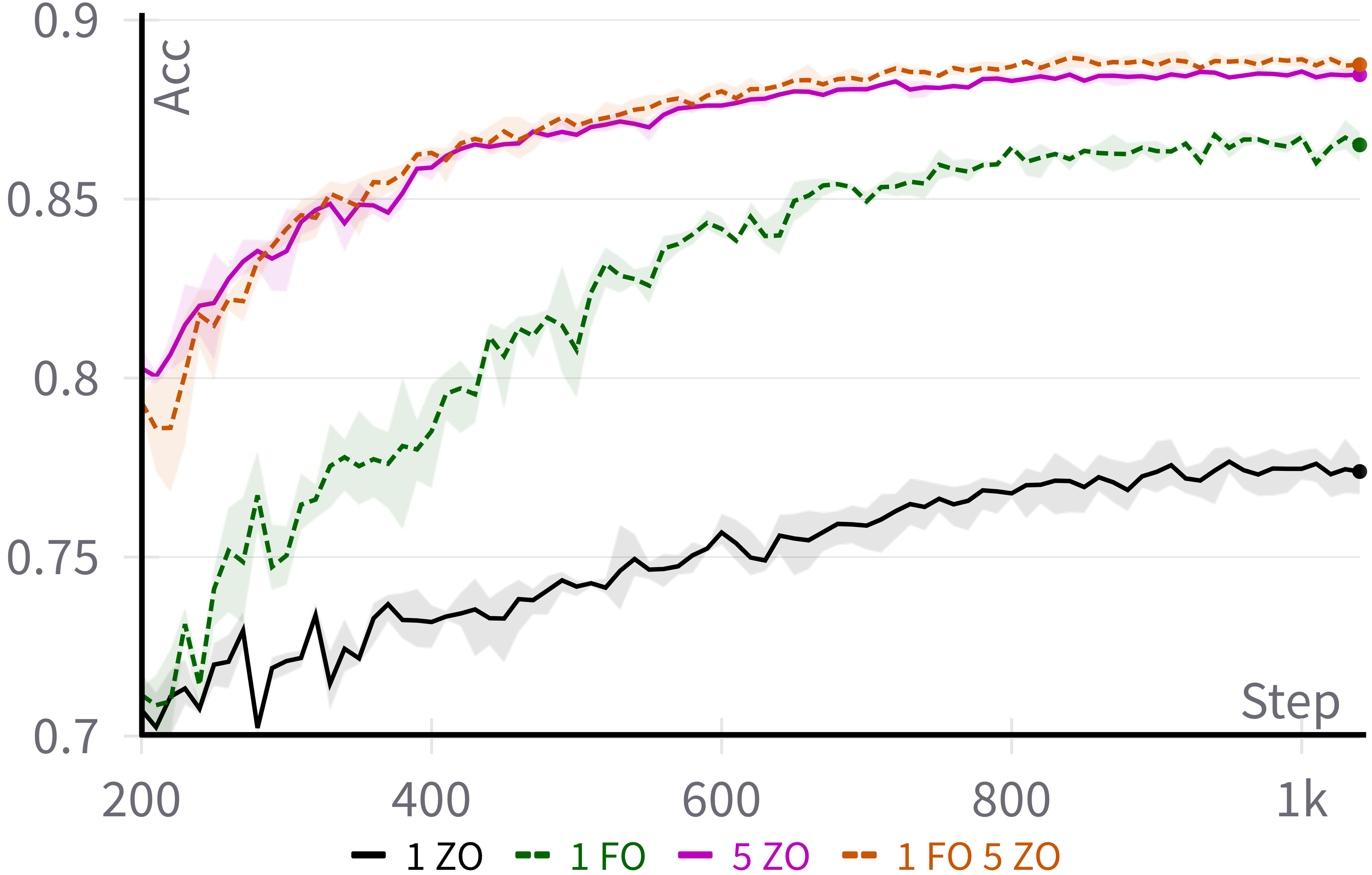} 
\caption{Validation accuracy (Acc) comparison between the hybrid and mono-type estimator population for ResNet-18 on the CIFAR-10 dataset.}
\label{fig:resnet_acc}
\end{figure}

\begin{figure}[t]
\centering
\includegraphics[width=1.0\columnwidth]{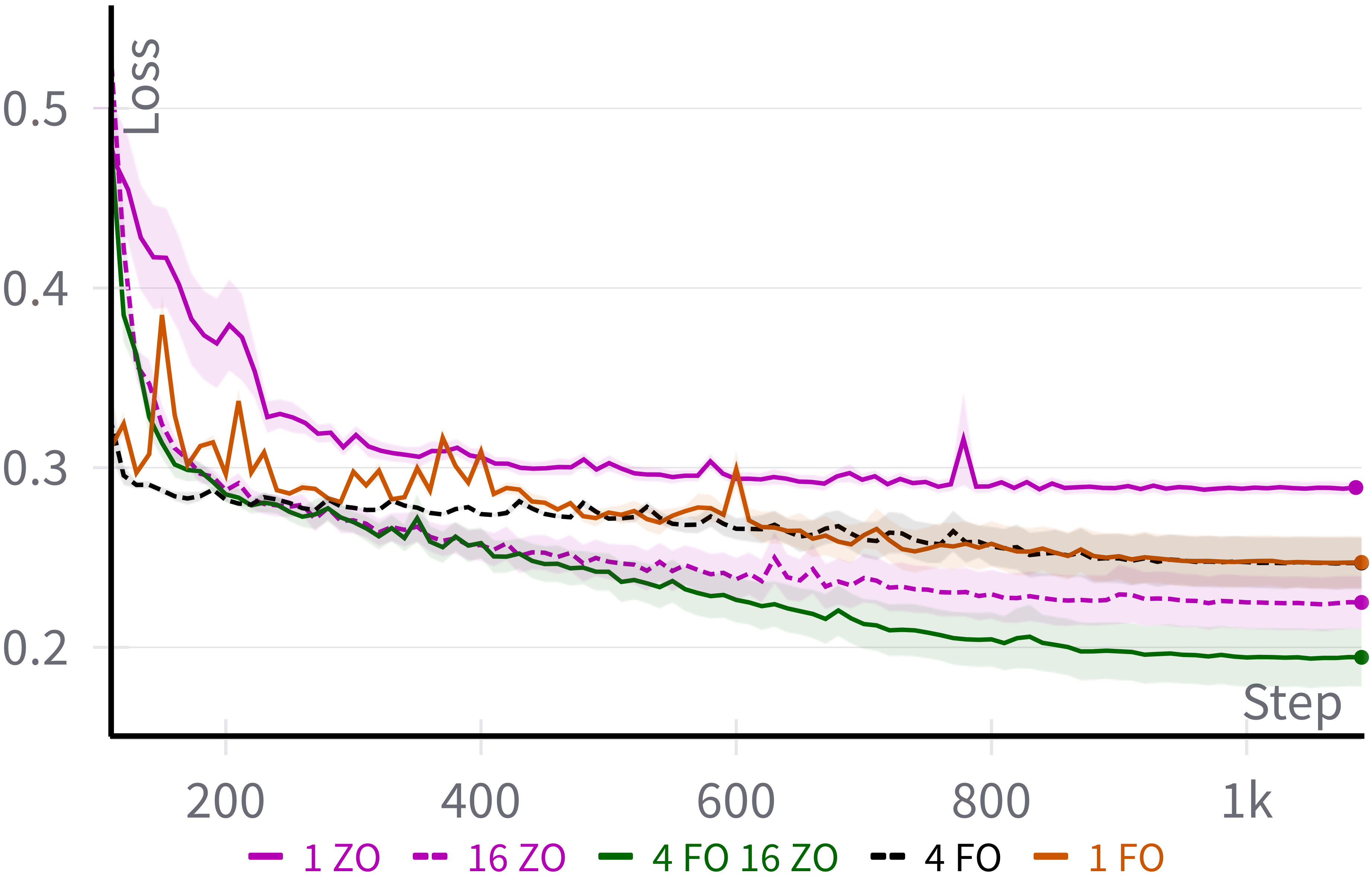}
\caption{Validation loss comparison between the hybrid and mono-type estimator population for transformer model on the synthetic Brackets dataset.}
\label{fig:transformer_loss}
\end{figure}

\paragraph{Results.} At first, we examine the performance of individual zeroth-order gradient estimators (with $\nu=10^{-4}$) over time, as a function of the number of random vectors (rv) used for the gradient estimation (Figure~\ref{fig:cnn_acc}); that is the number of $u$'s used to estimate the gradient using the equation~\ref{E(G_v)}. To do so, we use the MNIST classification task~\cite{deng2012mnist} using a Convolutional Neural Network (CNN)~\cite{726791} model. We choose values $8$, $16$, and $128$ for the number of random vectors, and compare against the unbiased forward-only estimator recently proposed by~\cite{baydin2022forwardmode}.  The results clearly demonstrate an accuracy-versus-steps advantage for a higher number of random vectors and for unbiased zeroth-order estimators compared to biased ones. Since the computational overhead of unbiasing estimators is relatively low~\cite{Chen2020unbiased}, we will use unbiased zeroth-order estimators in the subsequent experiments. A detailed explanation and experimental evaluation of the accuracy-efficiency trade-offs associated with the number of random vectors will be provided in the Appendix.


For the next set of experiments, we evaluate different \emph{mono-type populations} of ZO and FO optimizers and compare their performance with that of a hybrid population. In Figure~\ref{fig:resnet_acc}, we fine-tune the ResNet-18~\cite{he2015deepresiduallearningimage} model, pre-trained with ImageNet-1K~\cite{ILSVRC15}, on CIFAR-10~\cite{Krizhevsky2009LearningML} using populations of 1 ZO, 1 FO, 5 ZO, and a hybrid system consisting of 1 FO and 5 ZO nodes. To demonstrate the scalability of HDO in both convex and non-convex scenarios, we consider two settings: for the convex case, we assess the performance of a Logistic model on the MNIST dataset with various mono-type populations and a hybrid configuration of 24 FO and 256 ZO (Figure~\ref{fig:logistic_loss}). For the non-convex case, we use a Transformer model~\cite{vaswani2023attentionneed} on ``Brackets" dataset~\cite{ebrahimi2020selfattentionnetworksrecognizedyckn} (see the Appendix for more details). The populations that we study here are 1 ZO, 1 FO, 4 FO, 16 ZO, and a hybrid combination of 4 FO and 16 ZO (Figure~\ref{fig:transformer_loss}).

The results presented in Figure~\ref{fig:logistic_loss} (for the convex case), and Figures~\ref{fig:resnet_acc}~\&~\ref{fig:transformer_loss} (for the non-convex case) confirm the intuition, as well as our analysis; first-order nodes always outperform the same or lower number of zeroth-order ones, as shown in all the figures. However, zeroth-order nodes can in fact outperform first-order ones if their number is larger. We can conclude that 1) a larger population of ZO nodes can outperform smaller populations of FO or ZO nodes; and that 2) this larger uniform population is itself outperformed by a hybrid population.

Our experiments demonstrate the desired speedup and scalability; populations with a large number of ZO nodes have faster convergence. Interestingly, aligned with the Theorem~\ref{thm:main} that the speedup caused by the ZO nodes will appear after sufficiently large $T$, in Figure~\ref{fig:logistic_loss}, we observe that the (24-FO) group has a lower validation loss initially but the (256-ZO) and (24-FO, 256-ZO) groups outperform the former group after step 200. A similar phenomenon can also be observed in Figure~\ref{fig:transformer_loss}. These results validate the theoretical finding, showing that with a sufficient number of steps, hybrid populations achieve faster convergence compared to \emph{homogeneous populations} of first-order optimizers.

\section{Discussion, Limitations, and Future Work}

We provided a first analysis of the convergence of decentralized gradient-based methods in a population mixing first- and zeroth-order gradient estimators for both convex and non-convex objectives. Our results show that even biased or noisy zeroth-order information can enhance convergence when integrated into a protocol.

The experimental results validate our analysis and premise, demonstrating that first- and zeroth-order estimators can be effectively hybridized in a decentralized population. This is promising for environments with heterogeneous computational power, allowing agents to leverage local data even without gradient extraction capabilities.

A practical embodiment could be a decentralized learning system where computationally powerful agents perform backpropagation as first-order agents, while computationally limited nodes estimate gradients via forward passes over local data, sharing this information during pairwise interactions.

We focused on \emph{decentralized optimization}, where nodes interact in randomly chosen pairs. However, our analysis can extend to more general interaction graph topologies, where convergence depends on the eigenvalue gap. Another extension we plan to explore includes additional gradient estimators and large-scale practical deployments to validate our approach. Our experimental simulations confirm the feasibility of our method, achieving their intended goal.

\section{Acknowledgements}
This project has received funding from the European Research Council (ERC) under the European Union's Horizon 2020 research and innovation programme (grant agreement No 805223 ScaleML). The authors would like to acknowledge Eugenia Iofinova for useful discussions during the inception of this project. 

\bibliography{aaai25.bib}


\onecolumn
\normalfont
\appendix
\section{Appendix}

\section{Experimental Setup} \label{sec:experimental_setup}
In this section, we describe our experimental setup in detail. We begin by carefully describing the way in which we simulated HybridSGD in the sequential form. Then, we proceed by explaining the different types of gradient estimators that we used in our experiments, together with their implementation methods. Finally, we detail the datasets, tasks, and models used for our experiments. 

\subsection{Simulation}
We attempt to simulate a realistic decentralized deployment scenario sequentially,  as follows. We assume $n$ nodes, each of which initially has a model copy. Each node has an oracle to estimate the gradient of the loss function with respect to its model. It is assumed that $n_1$ nodes have access to first-order oracle and $n_0$ nodes have access to zeroth-order oracle, which can be biased or unbiased. The training dataset is distributed among first- and zeroth-order nodes so that each node has access to $\frac{1}{n}$ of training data. At each simulation step, we select $O(n)$ disjoint pairs uniformly at random and make each pair interact with each other. During the interaction, first, each node takes an SGD step and then they share their models and adapt the averaged model as their new model. To track the performance of our algorithm, each node evaluates its model on an unseen validation dataset which is shared between all the nodes. After every 10 steps, each node computes the validation \emph{loss} and \emph{accuracy}, and then the averaged loss and accuracy from all the nodes will be reported.

\subsection{Estimator types}

\begin{description}
    \item[First-order] Using this estimator node $i$ can estimate the gradient $\nabla f^i(X^i)$ by computing $\nabla F^i(X^i, \xi^i)$, where $f^i$ and $F^i$ are the node's local loss function and its stochastic estimator respectively. The computation is done using Pytorch built-in \textbf{.backward()} method. 
    \item[Unbiased Zeroth-order] We implemented this estimator using forward-mode differentiation technique inspired by \cite{}. Using this method, for a randomly chosen vector $u\sim N(0,I_d)$, node $i$ can compute $F^i(X^i, \xi^i)$ and $u.\nabla F^i(X^i, \xi^i)$ in a single forward pass. It will then use $(u.\nabla F^i(X^i, \xi^i))u$ as its gradient estimator. Note that the node does not need to compute $\nabla F^i(X^i, \xi^i)$, hence it is a zeroth-order estimation of the gradient. Moreover, $E_{u \sim N(0, I_d)}[(u.\nabla F^i(X^i, \xi^i))u] = \nabla F^i(X^i, \xi^i)$ which means the gradient estimator is unbiased.    
    \item[Biased Zeroth-order] For a fixed $\nu$ and a randomly chosen $u \sim N(0, I_d)$, node $i$ can estimate the gradient $\nabla F^i(X^i, \xi^i)$ simply by computing $\frac{F^i(X^i+\nu u, \xi^i) - F^i(X^i, \xi^i)}{\nu}u$ or $\frac{F^i(X^i+\nu u, \xi^i) - F^i(X^i-\nu u, \xi^i)}{2\nu}u$. The computations consist of evaluating only function values, thus they are called zeroth-order estimators. However, both of them are biased estimators as their expected values would be equal to the gradient of the smoothed-version of the function, $\nabla F^i_\nu(X^i, \xi^i)$, which is close but not necessarily equal to $\nabla F^i(X^i, \xi)$. 
\end{description}

Note that the approximation of zeroth-order estimators can be improved by increasing the number of randomly chosen vectors and averaging the results. For the biased zeroth-order estimators, we use the \emph{batch matrix multiplication} to compute the function values for all the randomly chosen vectors using constant GPU calls. Moreover, for unbiased zeroth-order estimators, we simulate the forward-mode differentiation by computing the gradient followed by computing the dot products of the gradient and the randomly chosen vectors. For more details on the implementation, we encourage readers to look at the source code of our experiments.

\subsection{Datasets and Models}
We use Pytorch to manage the training process in our algorithm, as well as Open MPI~\cite{gabriel04:_open_mpi} for communication between the nodes. In the first steps, we do some warm-up steps in a way that each node just trains its model without communication. We use a linear scheduler to increase the learning rate to the desired value during these steps. Then they start to communicate and we use a Cosine Annealing (CA) scheduler~\cite{loshchilov2017sgdrstochasticgradientdescent} to make more stable training between the nodes. We did the experiments in different random seeds and for each step, we computed the mean and standard error for the final evaluation. We use the Cross-Entropy loss function in our implementation. Since each step is training on a single batch, to decrease the noise of the computed gradient, we use a gradient momentum $g_{t+1} = mg_t + (1-m)\nabla G^i(x)$ where $m$ is the momentum value and $g_t$ is the gradient to update the model at step $t$. We fine-tuned the hyperparameters using grid search as much as possible. To be mentioned, we individually optimized the learning rates for a single agent of each node type, FO and ZO, to ensure optimal performance for both (recognizing that optimal rates are type-specific due to differences in gradient estimations). You can find the details of hyperparameter tuning in the tables below. Complete search was not possible because of the large search space. In the next paragraphs, we explain some task-specific details. 

\begin{description}
     \item[CNN on MNIST (Table~\ref{tab:ht_ab_study}).] The CNN model 
     used in this study consists of three convolutional and three linear layers. ReLU activation~\cite{Nair2010RectifiedLU} is applied throughout, with the convolutional layers having an output channel size of 8 and the linear layers having a width of 128. The experiments were conducted on a single A10 GPU with 24 GB of VRAM and 32 GB of system RAM.

     \item[ResNet-18 on CIFAR-10 (Table~\ref{tab:ht_resnet}).] ZOs have a lower memory footprint, which allows us to use a larger batch size. Since they process more data points simultaneously, it is more appropriate to consider a larger learning rate for them in certain regimes. This advantage of using a larger batch size is evident in theory, as it results in lower gradient variance for ZOs. Consequently, it would not be realistic to use the same batch size and learning rate for ZOs as for FOs, given their higher variance. By leveraging the lower memory footprint of ZOs to use larger batches, we are aligning the training process more closely with real-world scenarios and reducing variance. Failing to do so would mean not fully capitalizing on the benefits that ZOs offer. Therefore, after fine-tuning, we ended up using different learning rates and batch sizes for ZO and FO nodes. We conducted this setup on a single A10 GPU with 24 GB VRAM and 100 GB RAM.

     \item[Regression model on MNIST (Table~\ref{tab:ht_reg}).] We evaluated this scenario using 96 CPU cores and 4 RTX 3090 GPUs, each with 24 GB of VRAM. To maximize the number of nodes within our resource constraints, we used a single seed for this linear model. To better assess the behavior of ZO and FO nodes at scale, we set the minibatch size to 2. This choice prevents the models from seeing all the data in the initial steps, thereby avoiding premature convergence to the optimum point. Additionally, we omitted the use of gradient momentum and the CA scheduler to obtain a more realistic observation of the model's performance.

     \item[Transformer model on Brackets (Table~\ref{tab:ht_transformer}).] For this task, we employ a Transformer model with 2 layers of multi-head self-attention, each with 2 heads, and an embedding size of 4. A dropout rate of 0.1 is applied to mitigate overfitting. The experiments were conducted on a single Tesla T4 GPU with 15 GB of VRAM and 51 GB of system RAM.

     \item [Brackets.] This dataset consists of sequences of opening '(' and closing brackets ')'. The task is to predict the correctness of the entire sequence in terms of bracketing. A sequence is defined as correct if every opening bracket has a corresponding closing bracket. The training set consists of 25,600 samples, and the validation set consists of 2,560 samples. We use this dataset because the correct bracket sequences form a context-free language, which captures some properties of natural language~\cite{ebrahimi2020selfattentionnetworksrecognizedyckn}. Therefore, it can reveal nontrivial model capabilities while still being a relatively simple dataset.
\\
\\
\end{description}

\begin{minipage}{0.45\textwidth}
    \centering
    \begin{tabular}{|c|c|c|}
        \hline
        Hyperparameter & Value & Search interval \\
        \hline
        FO batch size & 256 & \{128, 256, 512\} \\
        \hline
        ZO batch size & 256 & \{128, 256, 512\} \\
        \hline
        FO learning rate & 0.01 & [0.001, 0.1] \\
        \hline
        ZO learning rate & 0.01 & [0.001, 0.1] \\
        \hline
        FO momentum & 0.9 & - \\
        \hline
        ZO momentum & 0.9 & - \\
        \hline
        T & 1000 & - \\
        \hline
        Number of seeds & 3 & - \\
        \hline
        Warm-up steps & 50 & - \\
        \hline
        CA scheduler & Yes & - \\
        \hline
        rv & - & - \\
        \hline
    \end{tabular}
    \captionof{table}{Hyperparameters-tuning details for studying the impact of the number of random vectors on the biased/unbiased ZO estimators, conducted using a CNN model on MNIST.}
    \label{tab:ht_ab_study}
\end{minipage}
\hfill
\begin{minipage}{0.45\textwidth}
    \centering
    \begin{tabular}{|c|c|c|}
        \hline
        Hyperparameter & Value & Search interval \\
        \hline
        FO batch size & 10 & [10, 100] \\
        \hline
        ZO batch size & 50 & [50, 250] \\
        \hline
        FO learning rate & 0.001 & [0.0001, 0.1] \\
        \hline
        ZO learning rate & 0.01 & [0.0001, 0.1] \\
        \hline
        FO momentum & 0.9 & - \\
        \hline
        ZO momentum & 0.9 & - \\
        \hline
        T & 1000 & - \\
        \hline
        Number of seeds & 3 & - \\
        \hline
        Warm-up steps & 50 & - \\
        \hline
        CA scheduler & Yes & - \\
        \hline
        rv & 128 & \{16, 32, 64, 128\} \\
        \hline
    \end{tabular}
    \captionof{table}{Hyperparameters-tuning details for the ResNet-18 model on CIFAR-10.}
    \label{tab:ht_resnet}
\end{minipage}

\vspace{1em} 

\begin{minipage}{0.45\textwidth}
    \centering
    \begin{tabular}{|c|c|c|}
        \hline
        Hyperparameter & Value & Search interval \\
        \hline
        FO batch size & 2 & - \\
        \hline
        ZO batch size & 2 & - \\
        \hline
        FO learning rate & 0.01 & [0.001, 0.1] \\
        \hline
        ZO learning rate & 0.01 & [0.001, 0.1] \\
        \hline
        FO momentum & - & - \\
        \hline
        ZO momentum & - & - \\
        \hline
        T & 500 & - \\
        \hline
        Number of seeds & 1 & - \\
        \hline
        Warm-up steps & 0 & - \\
        \hline
        CA scheduler & No & - \\
        \hline
        rv & 128 & \{32, 64, 128\} \\
        \hline
    \end{tabular}
    \captionof{table}{Hyperparameters-tuning details for the linear regression model on MNIST.}
    \label{tab:ht_reg}
\end{minipage}
\hfill
\begin{minipage}{0.45\textwidth}
    \centering
    \begin{tabular}{|c|c|c|}
        \hline
        Hyperparameter & Value & Search interval \\
        \hline
        FO batch size & 128 & \{128, 256, 1024, 2048\} \\
        \hline
        ZO batch size & 256 & \{128, 256, 1024, 2048\} \\
        \hline
        FO learning rate & 0.05 & [0.0001, 0.1] \\
        \hline
        ZO learning rate & 0.1 & [0.0001, 0.1] \\
        \hline
        FO momentum & 0.8 & [0.0, 0.95] \\
        \hline
        ZO momentum & 0.8 & [0.0, 0.95] \\
        \hline
        T & 1000 & - \\
        \hline
        Number of seeds & 17 & - \\
        \hline
        Warm-up steps & 100 & 100 \\
        \hline
        CA scheduler & Yes & - \\
        \hline
        rv & 64 & \{8, 16, 32, 64, 128\} \\
        \hline
    \end{tabular}
    \captionof{table}{Hyperparameters-tuning details to use the Transformer model on the Brackets dataset.}
    \label{tab:ht_transformer}
\end{minipage}

\section{More Ablation Studies} \label{sec:ab_study}
We pursue further ablation studies to cover more aspects of the theoretical results in our experiments.

\subsection{Learning Rate Impact} \label{subsec:lr_impact}
To analyze the impact of the learning rate (lr) on stochastic noise (Eq.~\ref{eq:noise}), we used the experimental settings described for the experiment of Table~\ref{tab:ht_reg}, with the following modifications: the experiments were conducted using three different random seeds, and the number of estimators was fixed at 90 ZO and 3 FO for all trials. The experiments were performed on A6000 GPUs with 48 GB VRAM each. The results, presented in Figure~\ref{fig:lr_ab}, show the relationship between learning rate and convergence behavior. The plot demonstrates that smaller learning rates (e.g., 0.005 and 0.01) result in smoother convergence curves, whereas larger learning rates (e.g., 0.5) exhibit more pronounced oscillations and slower convergence. 

\subsection{Effect of Random Vector Count on Convergence} \label{subsec:rv_count_effect}
While we used 64–128 random vectors (rv) to demonstrate effectiveness, using fewer random vectors can provide a balance between efficiency and accuracy. Experiments were conducted on the MNIST dataset using a multilayer perceptron (MLP) model with two hidden layers of size 128. The experiments were performed on A6000 GPUs with 48 GB VRAM, and additional details on hyperparameter tuning can be found in Table~\ref{tab:ht_low_rv}.

The results, shown in Figure~\ref{fig:low_rv}, illustrate the impact of the number of random vectors on loss convergence. Configurations with fewer random vectors (e.g., 8 or 16) show slower convergence and larger fluctuations in loss compared to configurations using higher numbers of random vectors (e.g., 32 or 64). This demonstrates that increasing the number of random vectors enhances the accuracy of gradient estimation, leading to more stable and rapid convergence. However, the computational overhead introduced by using more random vectors must be carefully balanced with available resources. The plot also underscores that zeroth-order nodes can achieve competitive performance by leveraging multiple forward passes, even with reduced computational capabilities.

\begin{figure}[ht]
    \centering
    \begin{minipage}[b]{0.45\textwidth}
        \centering
        \includegraphics[width=\textwidth]{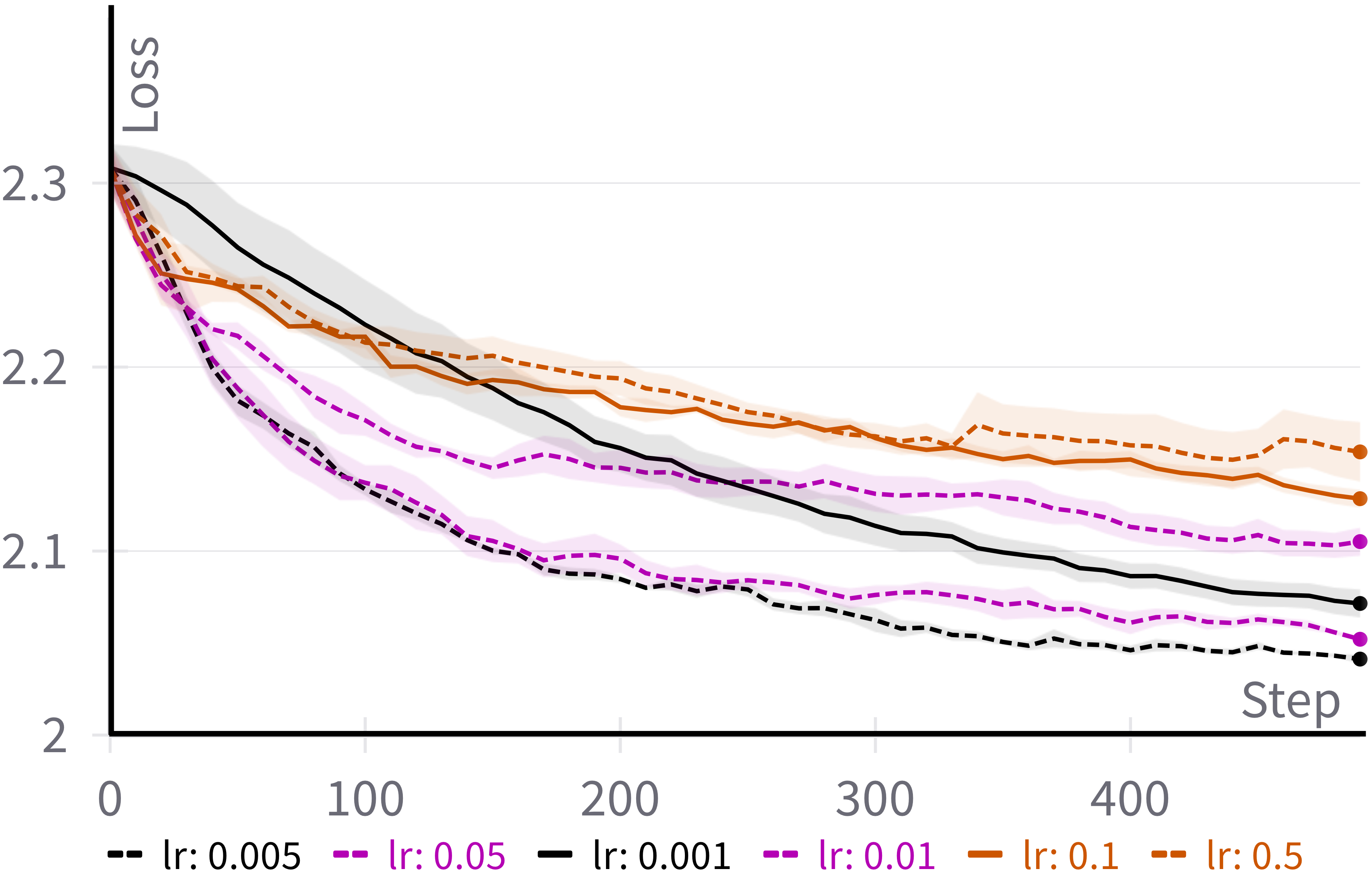}
        \caption{Impact of the learning rate (lr) on the validation loss, using a regression model on MNIST with 3 FO and 90 ZO nodes.}
        \label{fig:lr_ab}
    \end{minipage}
    \hfill 
    \begin{minipage}[b]{0.45\textwidth}
        \centering
        \includegraphics[width=\textwidth]{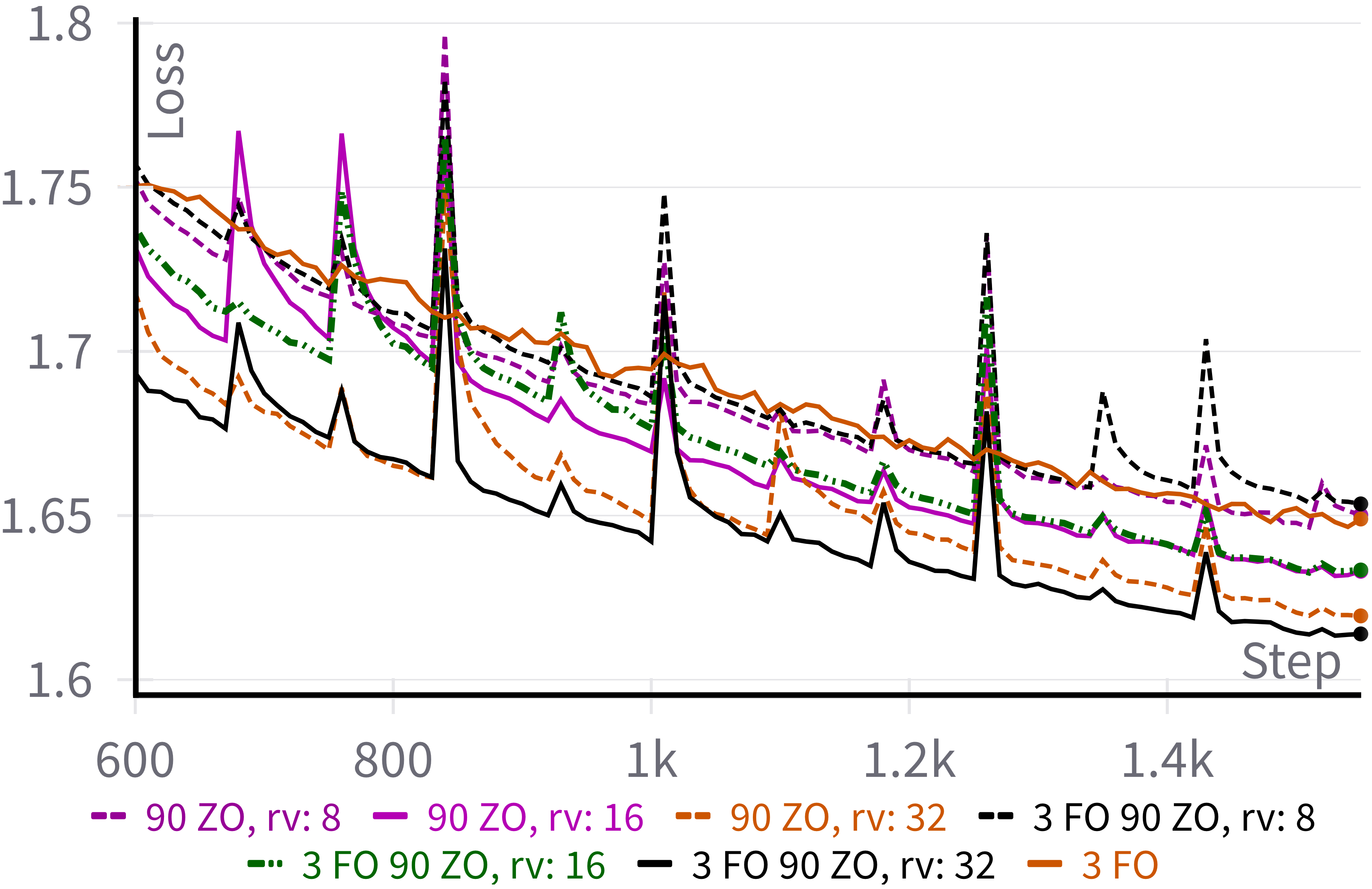}
        \caption{Impact of the different numbers of random vectors on the validation loss between the hybrid and mono-type estimator population, using an MLP model on MNIST. Confidence intervals are omitted for better interpretability.}
        \label{fig:low_rv}
    \end{minipage}
\end{figure}

\subsection{Model Consensus and ZO Population Impact} \label{subsec:consensus_zo_pop}
We expect that each node's model converges to a global model, likely due to each model converging toward a stable point in the final steps. This convergence leads to rapid population across models, as each step involves $O(n)$ interactions. To investigate this phenomenon, we conducted experiments with 16 nodes while varying the number of ZOs within the population to study the effect of ZOs on convergence. To assess the degree of convergence and the closeness of individual models to a global model, we measured the standard deviation of their losses.

The experiments were performed on the MNIST dataset using a model architecture comprising two convolutional layers and two linear layers. All experiments were conducted on A6000 GPUs with 48 GB of VRAM. Additional hyperparameter tuning details are provided in Table~\ref{tab:ht_con}, and the results are shown in Figure~\ref{fig:con_err}.

Figures~\ref{fig:con_err_loss} and~\ref{fig:con_err_std} demonstrate that the standard deviation of losses across models approaches zero in different settings, indicating strong convergence to a global model. Figure~\ref{fig:con_err_loss} shows that configurations with higher FO nodes (e.g., 16 FO) converge more rapidly, while configurations with more ZO nodes (e.g., 16 ZO) exhibit slower convergence but achieve a similar final loss. Meanwhile, Figure~\ref{fig:con_err_std} highlights that the standard deviation of losses diminishes across all configurations, showcasing consistent agreement among models regardless of the number of ZO nodes. This consistency supports the feasibility of using ZO nodes in heterogeneous systems while maintaining overall model consensus.

\begin{figure}[ht]
    \centering
    \begin{subfigure}[b]{0.45\textwidth}
        \includegraphics[width=\textwidth]{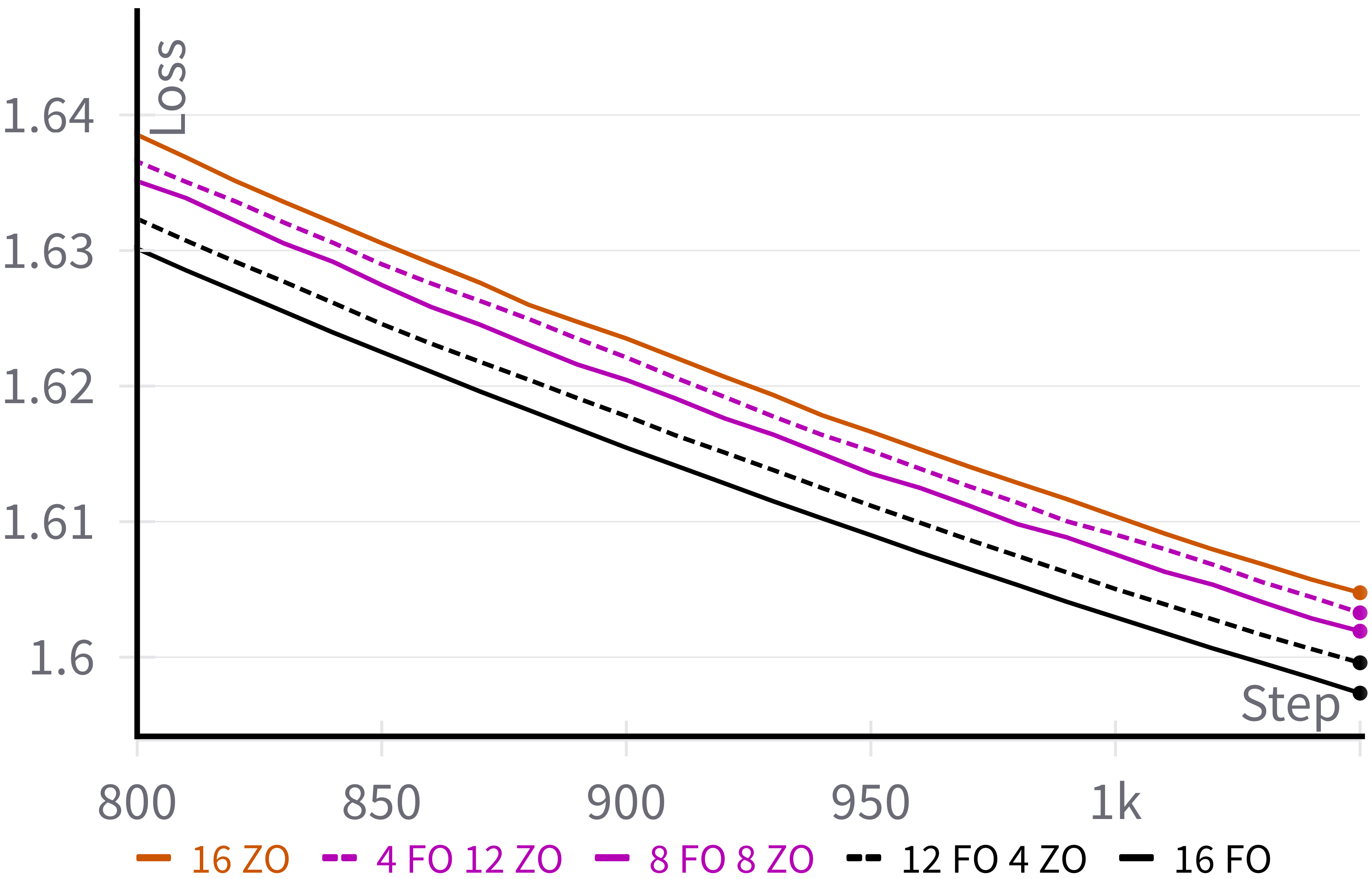}
        \caption{Validation loss vs. steps.}
        \label{fig:con_err_loss}
    \end{subfigure}
    \hfill 
    \begin{subfigure}[b]{0.45\textwidth}
        \includegraphics[width=\textwidth]{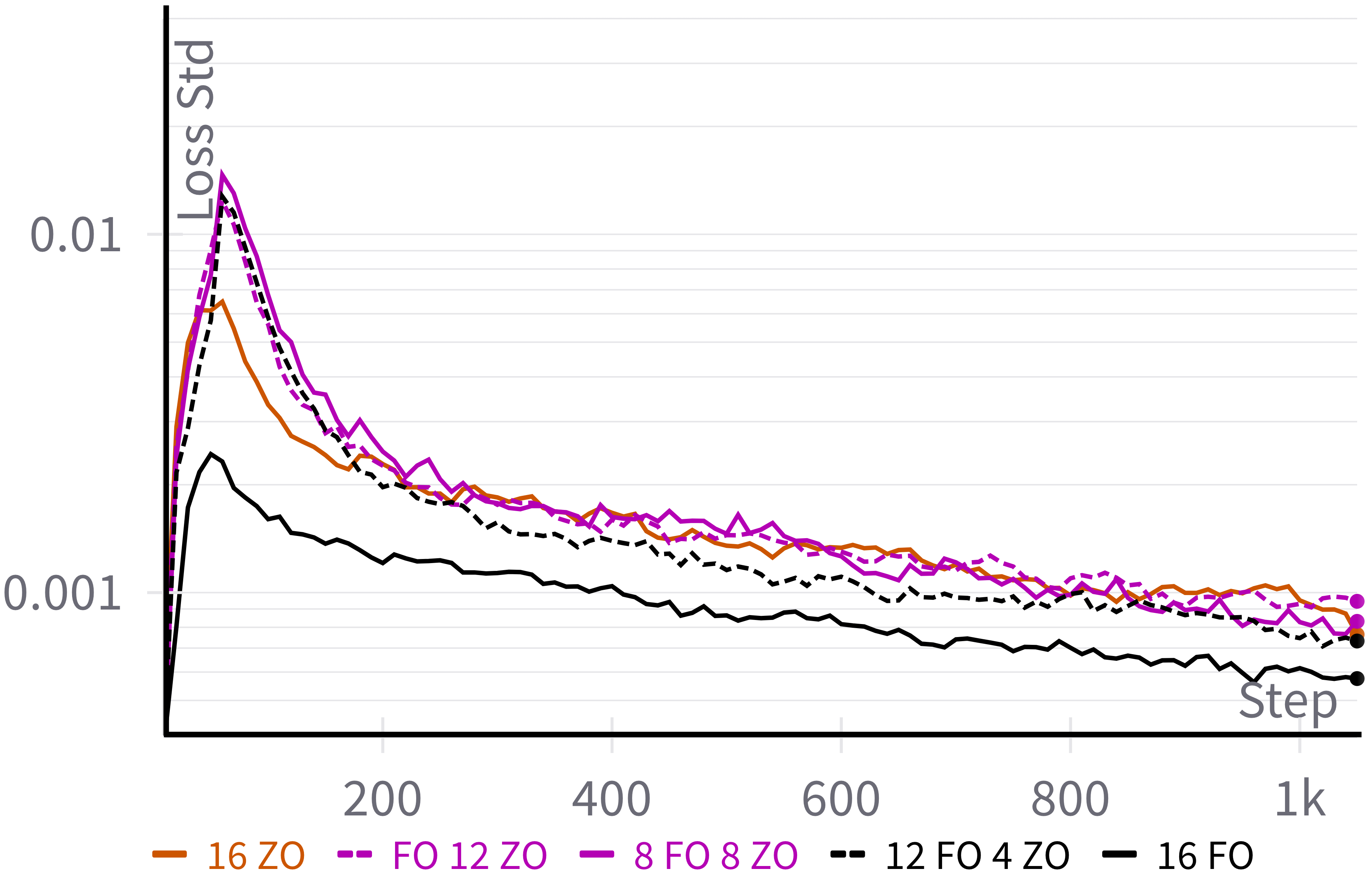}
        \caption{Loss Std (log scale) vs. steps.}
        \label{fig:con_err_std}
    \end{subfigure}
    \caption{(\ref{fig:con_err_loss}) Validation loss and (\ref{fig:con_err_std}) loss standard deviation across nodes' model loss (Loss Std), considering different populations of 16 nodes using a CNN model on MNIST.}
    \label{fig:con_err}
\end{figure}

\begin{minipage}{0.45\textwidth}
    \centering
    \begin{tabular}{|c|c|c|}
        \hline
        Hyperparameter & Value & Search interval \\
        \hline
        FO batch size & 8 & - \\
        \hline
        ZO batch size & 8 & - \\
        \hline
        FO learning rate & 0.1 & [0.001, 0.5] \\
        \hline
        ZO learning rate & 0.5 & [0.001, 0.5] \\
        \hline
        FO momentum & 0.0 & - \\
        \hline
        ZO momentum & 0.0 & - \\
        \hline
        T & 1500 & - \\
        \hline
        Number of seeds & 3 & - \\
        \hline
        Warm-up steps & 50 & - \\
        \hline
        CA scheduler & No & - \\
        \hline
        rv & - & - \\
        \hline
    \end{tabular}
    \captionof{table}{Hyperparameter tuning details for the ablation study on the impact of the number of random vectors on the convergence behavior of different populations, conducted using an MLP model on the MNIST dataset.}
    \label{tab:ht_low_rv}
\end{minipage}
\hfill
\begin{minipage}{0.45\textwidth}
    \centering
    \begin{tabular}{|c|c|c|}
        \hline
        Hyperparameter & Value & Search interval \\
        \hline
        FO batch size & 256 & \{8, 64, 256\} \\
        \hline
        ZO batch size & 256 & \{8, 64, 256\} \\
        \hline
        FO learning rate & 0.1 & [0.001, 0.1] \\
        \hline
        ZO learning rate & 0.1 & [0.001, 0.1] \\
        \hline
        FO momentum & 0.0 & - \\
        \hline
        ZO momentum & 0.0 & - \\
        \hline
        T & 1000 & - \\
        \hline
        Number of seeds & 3 & - \\
        \hline
        Warm-up steps & 50 & - \\
        \hline
        CA scheduler & No & - \\
        \hline
        rv & 128 & - \\
        \hline
    \end{tabular}
    \captionof{table}{Hyperparameter tuning details for the consensus study and analysis of the impact of different ZO populations, conducted using a CNN model on the MNIST dataset.}
    \label{tab:ht_con}
\end{minipage}

\section{Zeroth-order Stochastic Gradient Properties.}


\begin{lemma}\label{lem:Gv second moment upper bound}
Let $G^i_{\nu}(x, u, \xi^i)$ be computed by \ref{def: zeroth-order estimator}. Then, under Assumptions \ref{asmp:lipschitz} and \ref{asmp:unbiasedness_bounded_local_variance_of_F} we have:

\begin{equation}\label{eqn:Gv_second_moment_upper_bound}
\mathbb{E}_{u,\xi^i} \|G^i_{\nu}(x, u, \xi^i)\|^2 \leq \tfrac{1}{2} \nu^2 L^2 (d+6)^3+ 2 (d+4)  \left[\|\nabla f^i(x)\|^2+s_i^2\right], 
\end{equation}

\begin{equation} \label{eqn:Gv_variance_upper_bound}
    \mathbb{E}_{u, \xi^i}  \|G^i_\nu(x, u, \xi) - \nabla f^i(x) \|^2 \le \frac{3\nu^2}{2} L^2 (d+6)^3 + 4(d+4)\left[\|\nabla f^i(x)\|^2 + s_i^2\right].
\end{equation}
\end{lemma}

\begin{proof}
Firstly, by plugging in $F^i(x, \xi^i)$ in \ref{stoch_smth_approx_grad} under Assumptions \ref{asmp:lipschitz} and \ref{asmp:unbiasedness_bounded_local_variance_of_F}, we obtain

\begin{equation*}
\mathbb{E} \big\|G^i_{\nu}(x, u, \xi)\big\|^2 \leq \tfrac{1}{2} \nu^2 L^2 (d+6)^3+ 2 (d+4)  
\| \nabla F^i(x, \xi^i) \|^2
\end{equation*}
Then by getting an expectation and eliminating the randomness of the right-hand side with respect to $\xi^i$, we get

\begin{align*}
\mathbb{E}_{u, \xi^i} \big\|G^i_{\nu}(x, u, \xi^i)\big\|^2 &\leq \tfrac{1}{2} \nu^2 L^2 (d+6)^3+ 2 (d+4)  
\| \nabla F^i(x, \xi^i) \|^2 \\
&\overset{Assumption \ref{asmp:unbiasedness_bounded_local_variance_of_F}}{\leq} \tfrac{1}{2} \nu^2 L^2 (d+6)^3+ 2 (d+4)  \left[\|f^i(x)\|^2+s_i^2\right].
\end{align*}

Secondly, using \ref{E(G_v)} we have:
\begin{align*}\label{eqn:Gv - nabla fv}
    \E \big \| G_\nu(x, u, \xi) - \nabla f_\nu(x) \big \| ^2 &= \E \big\|G_{\nu}(x, u, \xi)\big\|^2 +  \big \| \nabla f_\nu(x) \big \|^2 - 2\langle \E \big(G_\nu(x, u, \xi) \big), \nabla f_\nu(x)\rangle\\
    &\overset{\ref{E(G_v)}}{=} \E \big\|G_{\nu}(x, u, \xi)\big\|^2 +
    \underbrace{ \big \| \nabla f_\nu(x) \big \|^2 -2 \big \| \nabla f_\nu(x) \big \|^2}_{- \big \| \nabla f_\nu(x) \big \|^2 \le 0}
    \le \E \big\|G_{\nu}(x, u, \xi)\big\|^2 \\
    &\overset{\ref{eqn:Gv_second_moment_upper_bound}}{\le} \tfrac{1}{2} \nu^2 L^2 (d+6)^3+ 2 (d+4)  \left[\|f^i(x)\|^2+s_i^2\right].
\end{align*}

Finally, together with Lemma~\ref{rand_smth_close_grad} and the inequality above we can deduce:
\begin{align*}
    \E_{u, \xi^i} \big \| G^i_\nu(x, u, \xi^i) - \nabla f^i(x) \| ^2 &\le 2\E_{u, \xi^i} \big \| G^i_\nu(x, u, \xi^i) - \nabla f^i_\nu(x) \big \|^2 + 2 \big \| \nabla f^i_\nu(x) - \nabla f^i(x) \big \|^2 \\
    &\overset{}{\le} \nu^2 L^2 (d+6)^3+ 4 (d+4)  \left[\|f^i(x)\|^2+s_i^2\right] + 2 \big \| \nabla f^i_\nu(x) - \nabla f^i(x) \big \|^2 \\
    &\overset{Lemma~\ref{rand_smth_close_grad}}{\le} \nu^2 L^2 (d+6)^3+ 4 (d+4)  \left[\|f^i(x)\|^2+s_i^2\right] + \frac{\nu^2}{2}L^2 (d+3)^3\\
    &\le \frac{3\nu^2}{2} L^2 (d+6)^3+ 4 (d+4)  \left[\|f^i(x)\|^2+s_i^2\right].
\end{align*}
\end{proof}

\section{Definitions}
For the sake of simplicity, we now define some notations for the frequently-used expressions in the proof.
\begin{definition}[Gamma]
\begin{equation}
 \Gamma_{t}  :=  \frac{1}{n}\sum_i\|X_{t}^i-\mu_{t}\|^2.
\end{equation}
\end{definition}

\begin{definition}[Average second-moment of estimator]
\begin{equation}
M_t^G := \frac{1}{n}\sum_i \big\|G^i(X_t^i)\big\|^2.
\end{equation}
\end{definition}

\begin{definition}[Expectation conditioned step]
\begin{equation}
\E_t[Y] := \E[Y|X^1_t,X^2_t,...,X^n_t].
\end{equation}
\end{definition}

\begin{definition}[Biasedness of estimators] For node $i$, using $G^i(x)$ as its gradient estimator we define $b_i$ as the upper bound for its biasedness, i.e.
\begin{equation}\label{def: beta_i}
    \big\|\nabla f^i(x) -  \E\big[G^i(x)\big]\big\| \leq b_i.
\end{equation}
Note that for an unbiased estimator we have $b_i = 0$. Moreover, for zeroth-order estimators $\E\big[G^i(x)\big]= \nabla f^i(x)$. Hence, according to Lemma~\ref{rand_smth_close_grad}, $\big\|\nabla f^i(x) -  \E\big[G^i(x)\big]\big\|$ is bounded for a fixed $\nu$. Therefore, $b_i$ is well-defined in our setup.\\
We further define the average biasedness of estimators as 
\begin{equation}\label{def: B}
    B:=\frac{1}{n} \sum_i b_i.
\end{equation}
\end{definition}

\begin{definition}[Variance of estimators] For node $i$, using $G^i(X_t^i)$ as its gradient estimator at step $t$, we define $(\sigma_t^i)^2$ as the upper-bound of its variance, i.e.
\begin{equation}
    E\big\|\nabla f^i(X_t^i) -  G^i(X_t^i)\big\|^2 \leq (\sigma_t^i)^2.
\end{equation}
Note that for the first-order nodes, i.e. $G^i(x)=\nabla F^i(x)$, using \ref{asmp:unbiasedness_bounded_local_variance_of_F} we have $(\sigma_t^i)^2:=s_i^2$. Moreover, for the zeroth-order nodes, i.e. $G^i(X^i_t)$ is computed using \ref{def: zeroth-order estimator}, according to \ref{eqn:Gv_variance_upper_bound} we have $(\sigma_t^i)^2:= \frac{3\nu^2}{2} L^2 (d+6)^3 + 4(d+4)\left[\|\nabla f^i(X_t^i)\|^2 + s_i^2\right]$, which is well-defined considering that $\nu$ is fixed in our setup.\\
We further define the average variance of estimators as
\begin{equation}\label{eqn:average variance}
    (\bar{\sigma}_t)^2:= \frac{1}{n}\sum_i(\sigma_t^i)^2.
\end{equation}
\end{definition}

\section{Useful Inequalities}
\label{useful-inequalities}
\begin{lemma}[Young] For any pair of vectors $x, y$ and $\alpha >0$ we have 
\begin{equation*}
    \langle x,y \rangle \leq \frac{\|x\|^2}{2\alpha} + \frac{\alpha \|y\|^2}{2}.
\end{equation*}
\end{lemma}
\begin{lemma}[Cauchy-Schwarz] For any vectors $x_1, x_2, ..., x_n \in \mathbb{}{R}^d$ we have
\begin{equation*}
    \| \sum_{i=1}^n x_i \|^2 \leq n \sum_{i=1}^n \|x_i\|^2.
\end{equation*}
\end{lemma}

\section{The Complete Convergence Proof}
\label{the-proof}

\subsection{Proof of Convex Case of Theorem \ref{thm:main}}

In this part, we assume that there exist $n_0$ zeroth-order nodes and $n_1$ first-order nodes, all having access to a shared dataset, hence a shared objective function $f$ that they want to minimize. 

\begin{lemma}\label{lem:h M^f}
For any time step $t$ and constants $\alpha_0 \ge \alpha_1 > 0$ let $M_t^f(\alpha_0, \alpha_1)=\frac{\alpha_0}{n}\sum_{i  \in N_0} \big\|\nabla f^i(X_t^i)\big\|^2 + \frac{\alpha_1}{n}\sum_{i \in N_1} \big\|\nabla f^i(X_t^i)\big\|^2$. We have that:
\begin{equation*}
 \E[M_t^f(\alpha_0, \alpha_1)] \le 
3L^2 \alpha_0 \E[\Gamma_t] + \frac{3\alpha_0 n_0\varsigma_0^2+3\alpha_1 n_1\varsigma_1^2}{n}+6 (\alpha_0+\alpha_1) L\E[f(\mu_t)-f(x^*)].
\end{equation*}
\end{lemma}
\begin{proof}
\begin{align}
\begin{split}
\frac{\alpha_0}{n}\sum_{i \in N_0} \E \big\|\nabla f^i(X_t^i)\big\|^2 
&= \frac{\alpha_0}{n}\sum_{i \in N_0} \E \big\|\nabla f(X_t^i) - \nabla f^i(\mu_t) + \nabla f^i(\mu_t)-\nabla f(\mu_t)+\nabla f(\mu_t) - \nabla f(x^*)\big\|^2\\
&\overset{\text{Assumptions }\ref{asmp:lipschitz} \text{ and } \ref{asmp:global_variance}, \text{Cauchy-Schwarz}}{\leq} 3L^2 \alpha_0 \sum_{i \in N_0}
\E \|X_t^i- \mu_t\|^2 + \frac{3\alpha_0 n_0\varsigma_0^2}{n}+6 \alpha_0 L\E[f(\mu_t)-f(x^*)].
\end{split}
\end{align}
Similarly, in the case of first-order nodes we get: 
\begin{align}
\begin{split}
\frac{\alpha_1}{n}\sum_{i \in N_1} \E \big\|\nabla f^i(X_t^i)\big\|^2 \le 
3L^2 \alpha_1 \sum_{i \in N_1}
\E \|X_t^i- \mu_t\|^2 + \frac{3\alpha_1 n_1\varsigma_1^2}{n}+6 \alpha_1 L\E[f(\mu_t)-f(x^*)].
\end{split}
\end{align}
By summing up the above inequalities and using the fact that $\alpha_0 \ge \alpha_1$ (together with the definition of $\Gamma_t$), we get the proof of the lemma.
\end{proof}

\MGBound*
\begin{proof}
\begin{align}
\begin{split}
\E_t\big[ M_t^G\big] &= \frac{1}{n}\sum_i \E_t \big\|G^i(X_t^i)\big\|^2 \overset{ (\ref{eqn:Gv_second_moment_upper_bound})}{\leq} \frac{1}{n}\sum_{i \in N_0} (\tfrac{1}{2} \nu^2 L^2 (d+6)^3+ 2 (d+4)  \left[\E_t\|\nabla f^i(X_t^i)\|^2+s_i^2\right]) \\&\quad\quad\quad\quad\quad\quad\quad\quad\quad\quad\quad\quad\quad\quad\quad+ \frac{1}{n} \sum_{i \in N_1}  (\E_t\|\nabla f^i(X_t^i)\|^2+s_i^2)\\
&\le \E_t[M_t^f(2(d+4),1)] + \frac{2(d+4) \sum_{i \in N_0} s_i^2+\sum_{i \in N_1} s_i^2}{n} + \eta^2 \frac{n_0}{2nc^2}L^2(d+6)^3.
\end{split}
\end{align}
Next, we take expectation with respect to $X_t^1, X_t^2, ..., X_t^n$ and use Lemma \ref{lem:h M^f} to get:
\begin{align}
\begin{split}
\E\big[ M_t^G\big] \leq 6(d+4)L^2 \E[\Gamma_t] &+ \frac{6(d+4)n_0\varsigma_0^2+3n_1 \varsigma_1^2}{n}+6(2d+9)L\E[f(\mu_t)-f(x^*)]\\&+ \frac{2(d+4) \sum_{i \in N_0} s_i^2+\sum_{i \in N_1} s_i^2}{n} + \eta^2 \frac{n_0}{2nc^2}L^2(d+6)^3.
\end{split}
\end{align}
Which finishes the proof of the lemma.
\end{proof}

\begin{lemma}\label{lem:h average variance upper bound}
Assume $\nu: =\frac{\eta}{c}$ is fixed, where $\eta$ and $c$ are the learning rate and a constant respectively. Then, for any time step t we have 
\begin{align*}
\E[(\bar{\sigma}_t)^2] \leq 12(d+4)L^2 \E[\Gamma_t] &+ \frac{12n_0(d+4)\varsigma_0^2+3n_1 \varsigma_1^2}{n}+6(4d+17)L(f(\mu_t) - f(x^*)) \\&+\frac{4(d+4)\sum_{i \in N_0} s_i^2+\sum_{i \in N_1} s_i^2}{n} + \eta^2 \frac{3 n_0}{2nc^2}L^2(d+6)^3.
\end{align*}
\end{lemma}
\begin{proof}
\begin{align}
\begin{split}
\E_t[(\bar{\sigma}_t)^2] &= \frac{1}{n}\sum_i\E_t[(\sigma_t^i)^2] \overset{(\ref{eqn:Gv_variance_upper_bound})}{\leq} \frac{1}{n}\sum_{i \in N_0}(\frac{3\nu^2}{2} L^2 (d+6)^3 + 4(d+4)\left[\E_t\|\nabla f(X_t^i)\|^2 + s_i^2\right]) 
\\&\quad\quad\quad\quad\quad\quad\quad\quad\quad+ \frac{1}{n} \sum_{i \in N_1} (\E_t\|\nabla f^i(X_t^i)\|^2+s_i^2)\\
&= \E_t[M^f_t(4(d+4),1)] + \frac{4(d+4)\sum_{i \in N_0} s_i^2+\sum_{i \in N_1} s_i^2}{n}\sigma^2 + \eta^2 \frac{3 n_0}{2nc^2}L^2(d+6)^3
. 
\end{split}
\end{align}
Next, we take expectation with respect to $X_t^1, X_t^2, ..., X_t^n$ and use Lemma \ref{lem:h M^f} to get:
\begin{align*}
\E_t[(\bar{\sigma}_t)^2] \le 12(d+4)L^2 \E[\Gamma_t] &+ \frac{12n_0(d+4)\varsigma_0^2+3n_1 \varsigma_1^2}{n}+6(4d+17)L(f(\mu_t) - f(x^*)) \\&+\frac{4(d+4)\sum_{i \in N_0} s_i^2+\sum_{i \in N_1} s_i^2}{n} + \eta^2 \frac{3 n_0}{2nc^2}L^2(d+6)^3. 
\end{align*}
Which finishes the proof of the lemma.
\end{proof}

\begin{lemma}\label{lem:h average biasedness upper bound}
Assume $\nu: =\frac{\eta}{c}$ is fixed, where $\eta$ and $c$ are the learning rate and a constant respectively. Then, for any time step t we have 
\begin{align*}
    B \le \eta \frac{n_0}{2cn}L(d+3)^{\frac{3}{2}}.
\end{align*}
\end{lemma}
\begin{proof}
\begin{align}
B &= \frac{1}{n}\sum_i b_i = \frac{1}{n}\sum_i \| \nabla f^i(X_t^i) - \E[G^i(X_t^i)]\| = \frac{1}{n}\sum_{i \in N_0} \| \nabla f^i(X_t^i) - \nabla f^i_\nu(X_t^i) \|\\
&\overset{(\ref{smth_approx})}{\le} \frac{\nu n_0}{2n}L(d+3)^{\frac{3}{2}} = \eta \frac{n_0}{2cn}L(d+3)^{\frac{3}{2}}.
\end{align}
\end{proof}

\GammaBoundPerStepHelper*
\begin{proof}
First we can open $\E_t\big[ \Gamma_{t+1} \big]$ as 
\begin{equation}
\E_t\big[ \Gamma_{t+1} \big] = \E_t\Big[ \frac{1}{n}\sum_i\|X_{t+1}^i-\mu_{t+1}\|^2\Big].
\end{equation}
Observe that in this case $\mu_{t+1}=\mu_t-\eta(G_{t}^i(X_t^i)+G_{t}^j(X_t^j))/n$ and\\ 
$X_{t+1}^i=X_{t+1}^j=(X_t^i+X_t^j)/2-\eta(G_{t}^i(X_t^i)+G_{t}^j(X_t^j))/2$. \\
Hence, 
\begin{align}\label{eq:Gamma_upperbound}
\begin{split}
&\E_t\big[ \Gamma_{t+1} \big] = \frac{1}{n^2(n-1)}\sum_i\sum_{i \neq j}\E_t\Bigg[ \begin{aligned}[t] &2\big\|(X_t^i+X_t^j)/2 - \big( \frac{n-2}{2n}\big)\eta(G^i(X_t^i) + G^j(X_t^j)) - \mu_t\big\|^2\\
&+\sum_{k \neq i, j}\big\|X_t^k - \mu_t + \frac{\eta}{n}(G^i(X_t^i) + G^j(X_t^j))\big\|^2\Bigg]\end{aligned}\\
&=\frac{1}{n^2(n-1)}\sum_i\sum_{i \neq j}\E_t\Bigg[2\Big(\begin{aligned}[t]& \|(X_t^i+X_t^j)/2 - \mu_t\|^2 + \big( \frac{n-2}{2n}\big)^2\eta^2\big\|G^i(X_t^i) + G^j(X_t^j)\big\|^2\\
&-\big( \frac{n-2}{n}\big)\eta\Big\langle G^i(X_t^i) + G^j(X_t^j), (X_t^i+X_t^j)/2 - \mu_t\Big\rangle\Big)\\
&+ \sum_{k \neq i, j}\Big( \|X_t^k - \mu_t\|^2 + \big(\frac{1}{n}\big)^2\eta^2\big\|G^i(X_t^i) + G^j(X_t^j)\big\|^2\\
&+\frac{2}{n}\eta\Big\langle G^i(X_t^i) + G^j(X_t^j), X_t^k-\mu_t\Big\rangle\Big)\Bigg]\end{aligned}\\
&=\frac{1}{n^2(n-1)}\sum_i\sum_{i \neq j}\E_t\Bigg[\begin{aligned}[t]&\sum_k\|X_t^k-\mu_t\|^2 - \|X_t^i-\mu_t\|^2/2 - \|X_t^j-\mu_t\|^2/2 + \langle X_t^i -\mu_t, X_t^j-\mu_t\rangle\\
&+\underbrace{\Big(\frac{(n-2)^2}{2n^2}+\frac{n-2}{n^2}\Big)}_{\frac{n-2}{2n}\leq\frac{1}{2}}\eta^2\big\|G^i(X_t^i) + G^j(X_t^j)\big\|^2\\
&-\underbrace{\Big(\frac{n-2}{n} + \frac{2}{n}\Big)}_{1}\eta\Big\langle G^i(X_t^i) + G^j(X_t^j), X_t^i+X_t^j - 2\mu_t\Big\rangle\Bigg]\end{aligned}\\
&\leq (1 - \frac{1}{n})\E_t\big[ \Gamma_t \big] -\frac{1}{n^2(n-1)}\eta\sum_i\sum_{i \neq j}\E_t \Big\langle G^i(X_t^i) + G^j(X_t^j), X_t^i+X_t^j - 2\mu_t\Big\rangle\\
&+\frac{1}{n^2(n-1)}\sum_i\sum_{i \neq j}\E_t \langle X_t^i -\mu_t, X_t^j-\mu_t\rangle +\frac{1}{2n^2(n-1)}\eta^2\sum_i\sum_{i \neq j} \E_t \big\| G^i(X_t^i)+G^j(X_t^j)\big\|^2\\
&\leq (1-\frac{1}{n})\E_t\big[ \Gamma_t \big] \begin{aligned}[t]& + \underbrace{\frac{1}{n^2(n-1)}\sum_i\sum_{i \neq j}\E_t \langle X_t^i -\mu_t, X_t^j-\mu_t\rangle}_{P_1:=} + \underbrace{\frac{2}{n^2}\eta^2 \sum_i \E_t\big\|G^i(X_t^i)\big\|^2}_{\frac{2}{n}\eta \E_t\big[M_t^G\big]}\\
&- \underbrace{\frac{1}{n^2(n-1)}\eta\sum_i\sum_{i \neq j}\E_t \Big\langle G^i(X_t^i) + G^j(X_t^j), X_t^i+X_t^j - 2\mu_t\Big\rangle}_{P_2:=} \end{aligned}.
\end{split}
\end{align}
Now we upper bound each of $P_1$ and $P_2$ as following
\begin{equation} \label{eq:P1}
P_1 = \frac{1}{n^2(n-1)}\sum_i\sum_{i \neq j}\E_t \langle X_t^i -\mu_t, X_t^j-\mu_t\rangle = \frac{-1}{n^2(n-1)}\sum_i\E_t\|X_t^i - \mu_t\|^2 = \frac{-1}{n(n-1)}\E_t\big[ \Gamma_t \big]
\end{equation}
\begin{align}\label{eq:P2}
\begin{split}
P_2 &= \frac{1}{n^2(n-1)}\eta\sum_i\sum_{i \neq j}\E_t \Big\langle G^i(X_t^i) + G^j(X_t^j), X_t^i+X_t^j - 2\mu_t\Big\rangle\\
&= \frac{2}{n^2(n-1)}\eta\Big( \sum_i\sum_{i \neq j} \E_t \Big\langle G^i(X_t^i), X_t^j - \mu_t\Big\rangle +(n-1)\sum_i\E_t\Big\langle G^i(X_t^i), X_t^i - \mu_t\Big\rangle\Big)\\
&=\frac{2(n-2)}{n^2(n-1)} \sum_i \E_t \Big\langle \eta G^i(X_t^i), X_t^i - \mu_t\Big\rangle \overset{\text{Young}}{\leq} \frac{1}{n^2}\sum_i\Big( 2\eta^2\E_t\big\|G^i(X_t^i)\big\|^2 + \frac{1}{2}\E_t\big\|X_t^i-\mu_t\big\|^2\Big)\\
& = \frac{2}{n}\eta^2\E_t\big[M_t^G\big] + \frac{1}{2n}\E_t \big[ \Gamma_t \big].
\end{split}
\end{align}
By using (\ref{eq:P1}) and (\ref{eq:P2}) in inequality (\ref{eq:Gamma_upperbound}) we get 
\begin{align}
\begin{split}
\E_t\big[ \Gamma_{t+1} \big] &\leq (1-\frac{1}{n})\E_t\big[ \Gamma_t \big] - \frac{1}{n(n-1)}\E_t\big[ \Gamma_t \big] + \frac{2}{n}\eta^2 \E_t\big[M_t^G\big] + \frac{2}{n}\eta^2\E_t\big[M_t^G\big] + \frac{1}{2n}\E_t \big[ \Gamma_t \big]\\
&\le \big( 1 - \frac{1}{2n}\big)\E_t\big[ \Gamma_t \big] + \frac{4}{n}\eta^2\E_t\big[M_t^G\big].
\end{split}
\end{align}
Finally, by taking the expectation with respect to $X_t^1, X_t^2, ..., X_t^n$ we will have
\begin{equation} \label{eqn:Gamma step bound with M^G}
    \E\big[ \Gamma_{t+1} \big] \leq \big( 1 - \frac{1}{2n}\big)\E\big[ \Gamma_t \big] + \frac{4}{n}\eta^2\E\big[M_t^G\big]
\end{equation}
\end{proof}

\begin{lemma} \label{lem:h GammaBoundPerStep}
For any time step $t$ and fixed learning rate $\eta \le \frac{1}{14 L (d+4)^\frac{1}{2}}$ 
\begin{align*}
\E\big[\Gamma_{t+1}\big] \leq \big( 1 - \frac{1}{4n})\E\big[ \Gamma_t \big] &+ \frac{12\eta^2(2(d+4)n_0\varsigma_0^2+n_1 \varsigma_1^2)}{n^2}+\frac{24\eta^2(2d+9)L\E[f(\mu_t)-f(x^*)]}{n} \\&+ \frac{4\eta^2(2(d+4) \sum_{i \in N_0} s_i^2+\sum_{i \in N_1} s_i^2)}{n^2} + \frac{2\eta^4 n_0 L^2(d+6)^3}{n^2c^2}.
\end{align*}
\end{lemma}
\begin{proof}
From Lemma \ref{lem:GammaBoundPerStepHelper} we get that:
\begin{equation*}
    \E\big[ \Gamma_{t+1} \big] \leq \big( 1 - \frac{1}{2n}\big)\E\big[ \Gamma_t \big] + \frac{4}{n}\eta^2\E\big[M_t^G\big].
\end{equation*}
Now, by using Lemma \ref{lem:h M^G} in the inequality above we have
\begin{align*}
    \E\big[ \Gamma_{t+1} \big] &\leq \big( 1 - \frac{1}{2n}\big)\E\big[ \Gamma_t \big] + \frac{4}{n}\eta^2\E\big[M_t^G\big]\\ \le& \big( 1 - \frac{1}{2n})\E\big[ \Gamma_t \big] + \frac{4\eta^2}{n}\Bigg(6(d+4)L^2 \Gamma_t+ \frac{6(d+4)n_0\varsigma_0^2+3n_1 \varsigma_1^2}{n}+6(2d+9)L(f(\mu_t)-f(x^*)) \\&\quad\quad\quad\quad\quad\quad\quad\quad\quad\quad\quad\quad+ \frac{2(d+4) \sum_{i \in N_0} s_i^2+\sum_{i \in N_1} s_i^2}{n} + \eta^2 \frac{n_0}{2nc^2}L^2(d+6)^3\Bigg) \\&=
    \big( 1 - \frac{1}{2n}+\frac{24\eta^2 L^2 (d+4)}{n})\E\big[ \Gamma_t \big] + \frac{12\eta^2(2(d+4)n_0\varsigma_0^2+n_1 \varsigma_1^2)}{n^2}+\frac{24\eta^2(2d+9)L\E[f(\mu_t)-f(x^*)]}{n} \\&\quad\quad\quad\quad\quad\quad\quad\quad\quad\quad\quad\quad+ \frac{4\eta^2(2(d+4) \sum_{i \in N_0} s_i^2+\sum_{i \in N_1} s_i^2)}{n^2} + \frac{2\eta^4 n_0 L^2(d+6)^3}{n^2c^2}.
\end{align*}
We get the proof of the lemma by using $\eta \le \frac{1}{14 L (d+4)^\frac{1}{2}}$ in the above inequality.
\end{proof}

Next, we define the following weights: for any step $t \ge 0$, let $w_t = (1-\frac{\eta \ell}{2n})^{-t}$.
This allows us to prove the following lemma:
\begin{lemma} \label{lem:h sum_gamma}
for any $T \ge 0$ and $\eta \le \frac{1}{10\ell}$:
\begin{align*}
\sum_{t=1}^{T} &w_t \E[\Gamma_{t-1}] \le 120\eta^2(2d+9)L \sum_{t=1}^{T-1} w_{t} \E[f(\mu_{t-1})-f(x^*)]  \\
&+ \Bigg(\frac{60\eta^2(2(d+4)n_0\varsigma_0^2+n_1 \varsigma_1^2)}{n}+\frac{20\eta^2(2(d+4) \sum_{i \in N_0} s_i^2+\sum_{i \in N_1} s_i^2)}{n} + \frac{10\eta^4 n_0 L^2(d+6)^3}{nc^2} \Bigg) \sum_{t=1}^{T-1} w_{t}.
\end{align*}
\end{lemma}

\begin{proof}
Let $P_t=\frac{24\eta^2(2d+9)L\E[f(\mu_t)-f(x^*)]}{n}$ \\ and let 
$Q=\frac{12\eta^2(2(d+4)n_0\varsigma_0^2+n_1 \varsigma_1^2)}{n^2}+\frac{4\eta^2(2(d+4) \sum_{i \in N_0} s_i^2+\sum_{i \in N_1} s_i^2)}{n^2} + \frac{2\eta^4 n_0 L^2(d+6)^3}{n^2c^2}.$ Then the above lemma gives us that for any $t \ge 0$: 
$\E[\Gamma_{t+1}] \le (1-\frac{1}{4n})\E[\Gamma_t]+P_t+Q$. After unrolling the recursion, we get that for any $t > 1$,
$
\E[\Gamma_t] \le \sum_{i=0}^{t-1} (P_i+Q)(1-\frac{1}{4n})^{t-1-i}$. Hence,
\begin{align}
\begin{split}
\label{eqn:h sum_weight_dot_gamma}
\sum_{t=1}^T w_t \E[\Gamma_{t-1}] &\le \sum_{t=2}^T w_t \Bigg(\sum_{i=0}^{t-2} (P_i+Q)(1-\frac{1}{4n})^{t-2-i} \Bigg) = \sum_{t=0}^{T-2} (P_t+Q) \sum_{i=t+2}^T w_i (1-\frac{1}{4n})^{i-2-t} \\ &=
(1-\frac{\eta\ell}{2n})^{-1} \sum_{t=0}^{T-2} (P_t+Q) \sum_{i=t+2}^T w_{t+1} (1-\frac{\eta \ell}{2n})^{-(i-(t+2))}(1-\frac{1}{4n})^{i-(t+2)}\\
&= (1-\frac{\eta\ell}{2n})^{-1} \sum_{t=0}^{T-2} w_{t+1}(P_t+Q) \sum_{j=0}^{T-(t+2)} \left( \frac{1-\frac{1}{4n}}{1-\frac{\eta \ell}{2n}}\right)^j.
\end{split}
\end{align}
For $\frac{1}{10\ell}\geq \eta$, we have $r:=\frac{1-\frac{1}{4n}}{1-\frac{\eta \ell}{2n}} \leq 1$. Hence, we can write
\begin{align}
    \sum_{j=0}^{T-(t+2)} \left( \frac{1-\frac{1}{4n}}{1-\frac{\eta \ell}{2n}}\right)^j = \sum_{j=0}^{T-(t+2)} r^j = \frac{1-r^{T-(t+1)}}{1-r} \overset{t \leq T-2}{\leq} \frac{1}{1-r}.
\end{align}
By using the above inequality  in (\ref{eqn:h sum_weight_dot_gamma}) we have
\begin{align*}
    \sum_{t=1}^T w_t \E[\Gamma_{t-1}] &\le (1-\frac{\eta\ell}{2n})^{-1}\frac{1}{1-\frac{1-\frac{1}{4n}}{1-\frac{\eta \ell}{2n}}} \sum_{t=0}^{T-2} w_{t+1}(P_t+Q)\\
    &= \frac{1}{\frac{1}{4n} - \frac{\eta \ell}{2n}} \sum_{t=0}^{T-2} w_{t+1}(P_t+Q)=\frac{1}{\frac{1}{4n} - \frac{\eta \ell}{2n}} \sum_{t=1}^{T-1} w_{t}(P_{t-1}+Q).
\end{align*}
Finally, since $\frac{1}{10\ell}\geq \eta$ we get $\frac{1}{\frac{1}{4n} - \frac{\eta \ell}{2n}} \leq 5n$ and the proof of lemma is finished. 
\end{proof}

\begin{lemma} \label{lem:h supmartingale}
For $\eta \le \frac{\sqrt{\ell c n}}{2\sqrt{Ln_0}(d+3)^\frac{3}{4}}$, we have that
\begin{align*}
\E \Big\| \mu_{t+1}  -x^* \Big \|^2 &\le (1-\frac{\ell\eta}{n} + \eta^2 \frac{4B}{n})\E\|\mu_t-x^*\|^2 - (4\frac{\eta}{n} - \eta^2\frac{16L(12d+52)}{n^2} - \eta^4 \frac{64BL}{n^3})\E\big[f(\mu_t)-f(x^*)\big]\\
&+ \left( 2\frac{L+\ell}{n}\eta + \eta^2L^2 \frac{96d+456}{n^2} + \eta^4 \frac{32BL^2}{n^3}\right)\E[\Gamma_t]\\
&+ \frac{\eta^2((96d+448)\varsigma_0^2 n_0+88\varsigma_1^2 n_1)}{n^3}+\frac{8\eta^2((d+4)\sum_{i \in N_0} s_i^2+\sum_{i \in N_1} s_i^2)}{n^3}+\frac{12\eta^4 n_0 L^2 (d+6)^3}{n^3c^2}+\eta^2 \frac{2B}{n}.
\end{align*}
\end{lemma}
\begin{proof}
Let $F_t$ be the amount by which $\mu_t$ decreases at step $t$.
So, $F_t$ is a sum of $\frac{\eta}{n}G^i(X_t^i)$ and $\frac{\eta}{n}G^j(X_t^j)$ for agents $i$ and $j$, 
which interact at step $t$.
Also, let $F'_t$ be the amount by which $\mu_t$ would decrease
if all the agents were contributing at that step using their true local gradients. 
That is $F'_t = \frac{2\eta}{n^2}\sum_i \nabla f^i(X_t^i)$.\\
To make the calculations more clear, lets define $\E_t[Y] := \E[Y|X^1_t,X^2_t,...,X^n_t]$.

\begin{subequations}
\label{eq:h NewDecNorm}
\begin{align*} 
\E \Big \| \mu_{t+1}  -x^* \Big \|^2 
&= \E \Big \| \mu_t-F_t-x^* \Big \|^2 = \E \Big \| \mu_t-F_t-x^* - F'_t+F'_t  \Big \|^2  \\
&=\E \Big \|\mu_t-x^*-F'_t \Big\|^2+ \E \Big\|F'_t-F_t \Big\|^2 + 2 \E \Big \langle \mu_t-x^*-  F'_t, F'_t-F_t \Big \rangle\\
&=\E \Big \|\mu_t-x^*-F'_t \Big\|^2+ \E_{X^1_t,X^2_t,...,X^n_t} \Big[ \E_t \Big\|F'_t-F_t \Big\|^2\Big] \\
&+ 2 \E_{X^1_t,X^2_t,...,X^n_t} \Big[ \E_t \Big \langle \mu_t-x^*-  F'_t, F'_t-F_t \Big \rangle\Big]\\
\tag{\ref{eq:h NewDecNorm}}
\end{align*}
\end{subequations}

This means that in order to upper bound $\E \Big \| \mu_{t+1}  -x^* \Big \|^2$,
we need to upper bound $\E \Big \|\mu_t-x^*- F'_t \Big\|^2$, $\E_t \Big\|F'_t-F_t \Big\|^2$, and $\E_t \Big \langle \mu_t-x^*-  F'_t, F'_t-F_t \Big \rangle$.

For the first one, when $X_1, X_2, ..., X_n$ are fixed, we have that 

\begin{equation}
\label{eq:NewDecNormTrueGradient}
\begin{split}
\Big \|\mu_t-x^*- F'_t \Big\|^2 &= \Big \|\mu_t-x^*- \frac{2\eta}{n^2} \sum_i \nabla f^i(X_t^i)\Big\|^2\\
&= \big\|\mu_t-x^*\big\|^2 + 4\frac{\eta^2}{n^2} \underbrace{\Big\|\frac{1}{n} \sum_i \nabla f^i(X_t^i)\Big\|^2}_{R_1:=} -4\frac{\eta}{n} \underbrace{\Big\langle \mu_t-x^*, \frac{1}{n} \sum_i \nabla f^i(X_t^i) \Big\rangle}_{R_2:=}
\end{split}
\end{equation}

\begin{align} \label{eq:R1}
\begin{split}
R_1 &= \Big\|\frac{1}{n}\sum_i \nabla f^i(X_t^i)\Big\|^2 = \Big\|\frac{1}{n} \sum_i \nabla f^i(X_t^i) -  \nabla f^i(\mu_t) + \nabla f^i(\mu_t) - \nabla f^i(x^*)\Big\|^2\\
&\overset{\text{Cauchy-Schwarz}}{\leq} \frac{2}{n} \sum_i \Big\|\nabla f^i(X_t^i) - \nabla f^i(\mu_t)\Big\|^2 + 2\Big\| \frac{1}{n} \sum_i \nabla f^i(\mu_t) - \nabla f^i(x^*) \Big\|^2\\
&\leq \frac{2L^2}{n}\sum_i \big\|X_t^i - \mu_t\big\|^2 + \frac{4L}{n}\sum_i \big(f^i(\mu_t) - f^i(x^*)\big)\\
&= \frac{2L^2}{n}\sum_i \big\|X_t^i - \mu_t\big\|^2 + 4L\big[f(\mu_t) - f(x^*) \big]
\end{split}
\end{align}

\begin{subequations}
\begin{align}\label{eq:B_decomposition}
\begin{split}
R_2 &= \Big\langle \mu_t-x^*, \frac{1}{n} \sum_i \nabla f^i(X_t^i) \Big\rangle = \frac{1}{n} \sum_i \Big\langle \mu_t-X_t^i+X_t^i-x^*, \nabla f_t^i(X_t^i) \Big\rangle\\
&= \frac{1}{n} \sum_i \Big[\Big\langle \mu_t-X_t^i, \nabla f^i(X_t^i) \Big\rangle + \Big\langle X_t^i-x^*, \nabla f^i(X_t^i) \Big\rangle \Big]
\end{split}
\end{align}
Using L-smoothness property (Assumption \ref{asmp:lipschitz}) with $y=X_t^i$ and $x=x^*$ we have
\begin{equation}\label{eq:B1_lowerbound}
\Big\langle \mu_t-X_t^i, \nabla f^i(X_t^i) \Big\rangle \geq f^i(\mu_t) - f^i(X_t^i) - \frac{L}{2}\| \nabla f^i(\mu_t) - \nabla f^i(X_t^i) \|^2. 
\end{equation}
Additionally, we use the $\ell$-strong convexity (Assumption \ref{asmp:strongly_convex}), to get
\begin{equation}\label{eq:B2_lowerbound}
\Big\langle X_t^i-x^*, \nabla f^i(X_t^i) \Big\rangle \geq (f^i(X_t^i) - f^i(x^*)) + \frac{\ell}{2}\|X_t^i-x^*\|^2.
\end{equation}
\end{subequations}
Now by plugging  (\ref{eq:B1_lowerbound}) and (\ref{eq:B2_lowerbound}) in inequality (\ref{eq:B_decomposition}) we get that
\begin{align}\label{eq:R2}
\begin{split}
R_2 &\geq \frac{1}{n} \sum_i \Big[ f^i(\mu_t) - f^i(X_t^i) - \frac{L}{2}\| \nabla f^i(\mu_t) - \nabla f^i(X_t^i) \|^2 + f^i(X_t^i) - f^i(x^*) + \frac{\ell}{2}\|X_t^i-x^*\|^2 \Big]\\
&= \big[f(\mu_t)-f(x^*)] - \frac{L}{2n}\sum_i \|X_t^i - \mu_t\|^2 + \frac{\ell}{2n}\sum_i\|X_t^i - x^*\|^2\\
&\geq \big[f(\mu_t)-f(x^*)] - \frac{L+\ell}{2n}\sum_i \big\|X_t^i - \mu_t\big\|^2 + \frac{\ell}{4}\|\mu_t - x^*\|^2.
\end{split}
\end{align}
Now we plug (\ref{eq:R1}) and (\ref{eq:R2}) back into (\ref{eq:NewDecNormTrueGradient}) and take expectation into the account to get
\begin{align}\label{eq:h first_term}
\begin{split}
&\E \Big \|\mu_t-x^*- F'_t \Big\|^2 \begin{aligned}[t]&\le \E \big\|\mu_t-x^*\big\|^2 + 4\frac{\eta^2}{n^2} \Big( \frac{2L^2}{n}\sum_i \big\|X_t^i - \mu_t\big\|^2 + 4L\E \big[f(\mu_t) - f(x^*) \big]\Big)\\
&-4\frac{\eta}{n} \Big( \big[f(\mu_t)-f(x^*)\big] - \frac{L+\ell}{2n}\sum_i \E \big\|X_t^i - \mu_t\big\|^2 + \frac{\ell}{4}\E \|\mu_t - x^*\|^2 \Big)\end{aligned}\\
&= (1-\frac{\ell\eta}{n})\E \|\mu_t-x^*\|^2 - (4\frac{\eta}{n}-16L\frac{\eta^2}{n^2})\E \big[f(\mu_t)-f(x^*)\big] + \big(2\frac{L+\ell}{n}\eta + 8\frac{L^2}{n^2}\eta^2\big)\E[\Gamma_t].
\end{split}
\end{align}

For the second one we have that:

\begin{align} \label{eq:h second_term}
\begin{split}
&\E_t \Big\|F'_t-F_t \Big\|^2 = \frac{1}{n(n-1)}\sum_i\sum_{i \neq j} \E_t\Big\|\frac{2\eta}{n^2}\sum_r\nabla f^r(X_t^r) - \frac{\eta}{n}(G^i(X_t^i) + G^j(X_t^j))\Big\|^2\\
&\leq \frac{4\eta^2}{n^3}\sum_i \E_t\Big\|\frac{1}{n}\sum_r\nabla f^r(X_t^r) - G^i(X_t^i)\Big\|^2 = \frac{4\eta^2}{n^3}\sum_i \E_t\Big\|\frac{1}{n}\sum_r\nabla f^r(X_t^r) - \nabla f^i(X_t^i) + \nabla f^i(X_t^i) - G^i(X_t^i)\Big\|^2\\
&\leq \frac{8\eta^2}{n^3}\sum_i \Big(\E_t\Big\|\frac{1}{n}\sum_r\nabla f^r(X_t^r) - \nabla f^i(X_t^i)\Big\|^2 + \E_t[(\sigma_t^i)^2]  \Big) \\ &\leq \frac{8\eta^2}{n^3}\Big(\sum_i \E_t\Big\|\frac{1}{n-1}\sum_{r \neq i} [\nabla f^r(X_t^r) - \nabla f^i(X_t^i)]\Big\|^2 + \E_t[(\sigma_t^i)^2] \Big)\\
&\leq \frac{8\eta^2}{n^3(n-1)}  \sum_i \sum_{r \neq i}\E_t\Big\| \nabla f^r(X_t^r) - \nabla f^i(X_t^i)\Big\|^2 + \frac{8\eta^2}{n^2}\E_t[(\bar{\sigma}_t)^2] \\
&\leq \frac{8\eta^2}{n^3(n-1)} \sum_i \sum_{r \neq i}\E_t\Big\| \begin{aligned}[t] &[\nabla f^r(X_t^r) - \nabla f^r(\mu_t)] + [\nabla f^r(\mu_t) - \nabla f(\mu_t)] \\ &+ [\nabla f(\mu_t) - \nabla f^i(\mu_t)] + [\nabla f^i(\mu_t) - \nabla f^i(X_t^i)]\Big\|^2 + \frac{8\eta^2}{n^2}\E_t[(\bar{\sigma}_t)^2] \end{aligned}\\
&\leq \frac{8\eta^2}{n^3(n-1)} \sum_i 8(n-1) \Big(\E_t\Big\|\nabla f^i(X_t^i) - \nabla f^i(\mu_t) \Big\|^2+ \E_t\Big\|\nabla f^i(\mu_t) - \nabla f(\mu_t)\Big\|^2\Big) + \frac{8\eta^2}{n^2}\E_t[(\bar{\sigma}_t)^2]\\
&\leq \frac{64\eta^2}{n^3} \sum_i \E_t\Big\|\nabla f^i(X_t^i) - \nabla f^i(\mu_t) \Big\|^2 + \frac{64\eta^2}{n^3} \sum_i \E_t\Big\|\nabla f^i(\mu_t) - \nabla f(\mu_t)\Big\|^2 + \frac{8\eta^2}{n^2}\E_t[(\bar{\sigma}_t)^2]\\
&\leq \frac{64L^2\eta^2}{n^3} \sum_i \E_t\|X_t^i - \mu_t\|^2 + \frac{64\eta^2(\varsigma_0^2 n_0+\varsigma_1^2 n_1)}{n^3}+ \frac{8\eta^2}{n^2}\E_t[(\bar{\sigma}_t)^2] \\ &= \frac{64L^2\eta^2}{n^2} \E_t[\Gamma_t] + \frac{64\eta^2(\varsigma_0^2 n_0+\varsigma_1^2 n_1)}{n^3}+ \frac{8\eta^2}{n^2}\E_t[(\bar{\sigma}_t)^2].\\
\end{split}
\end{align}

Next, we remove conditioning and use Lemma \ref{lem:h average variance upper bound} to get
\begin{align}
\begin{split}
\E \Big\|F'_t-F_t \Big\|^2 &=\E\Bigg[\E_t \Big\|F'_t-F_t \Big\|^2\Bigg] \le \frac{64L^2\eta^2}{n^2} \E[\Gamma_t] + \frac{64\eta^2(\varsigma_0^2 n_0+\varsigma_1^2 n_1)}{n^3} + \frac{8\eta^2}{n^2}\E[(\bar{\sigma}_t)]^2 \\&\le
\frac{64L^2\eta^2}{n^2} \E[\Gamma_t] + \frac{64\eta^2(\varsigma_0^2 n_0+\varsigma_1^2 n_1)}{n^3} \\&\quad\quad\quad\quad\quad+ \frac{8\eta^2}{n^2}\Bigg(12(d+4)L^2 \E[\Gamma_t] + \frac{12n_0(d+4)\varsigma_0^2+3n_1 \varsigma_1^2}{n}+6(4d+17)L(f(\mu_t) - f(x^*)) \\&\quad\quad\quad\quad\quad\quad\quad\quad\quad\quad+\frac{4(d+4)\sum_{i \in N_0} s_i^2+\sum_{i \in N_1} s_i^2}{n} + \eta^2 \frac{3 n_0}{2nc^2}L^2(d+6)^3 \Bigg) \\ &=
\frac{\eta^2L^2(96d+448)}{n^2} \E[\Gamma_t] + \frac{\eta^2((96d+448)\varsigma_0^2 n_0+88\varsigma_1^2 n_1)}{n^3} \\&\quad\quad\quad+ \frac{48\eta^2(4d+17)L\E[f(\mu_t) - f(x^*)]}{n^2} +\frac{8\eta^2((d+4)\sum_{i \in N_0} s_i^2+\sum_{i \in N_1} s_i^2)}{n^3} + \frac{12\eta^4 n_0 L^2 (d+6)^3}{n^3c^2}.
\end{split}
\end{align}
Now consider the last one. We have:
\begin{align}\label{eq:h third_term_1}
\begin{split}
\E_t \Big \langle \mu_t-x^*-  F'_t, F'_t-F_t \Big \rangle & =
\Big \langle  \mu_{t} - F'_t - x^*,  \E_t(F'_t - F_t) \Big \rangle \\
& \overset{\text{Cauchy-Schwarz}}{\leq} \|\mu_{t} - x^*  - F'_t\|  \cdot \big\|\E_t(F'_t-F_t)\big\| \\
& \leq \Big[\|\mu_{t} - x^* \| + \|F'_t\|\Big]  \cdot \underbrace{\big\|\E_t(F'_t-F_t)\big\|}_{R_3:=}.
\end{split}
\end{align}

\begin{align}\label{eq:h R3}
\begin{split}
R_3&=\big\|\E_t(F'_t-F_t)\big\| = \big\|\frac{2\eta}{n^2}\sum_i\nabla f^i(X_t^i) - \frac{2\eta}{n^2} \E_t\big[G^i(X_t^i)\big]\big\|\\
& \leq \frac{2\eta}{n^2} \sum_i \big\|\nabla f^i(X_t^i) -  \E_t\big[G^i(X_t^i)\big]\big\| \leq \frac{2\eta^2}{n^2} \sum_i b_i \overset{\ref{def: B}}{=} \frac{2\eta^2}{n}B 
\end{split}
\end{align}
By using the inequality above in (\ref{eq:h third_term_1}) and taking expectation from both sides we get
\begin{align}\label{eq:h third_term}
\begin{split}
\E &\Big\langle \mu_t-x^* - F'_t, F'_t-F_t \Big\rangle = \E_{X^1_t,X^2_t,...,X^n_t} \bigg[ \E_t \Big\langle \mu_t-x^* - F'_t, F'_t-F_t \Big\rangle\bigg]\\
&\leq \frac{2\eta^2}{n}B \big( \E\|\mu_{t} - x^* \| + \E\|F'_t\| \big) \overset{(*)}{\leq} 2\frac{\eta^2}{n}B\big( \E\|\mu_{t}-x^*\|^2 + \tfrac{1}{4} + \E\|F'_t\|^2 + \tfrac{1}{4} \big)\\
&\overset{(\ref{eq:R1})}{\leq} 2\frac{\eta^2}{n}B\big( \E\|\mu_{t}-x^*\|^2 + \frac{4\eta^2}{n^2}\E\Big\|\frac{1}{n}\sum_i \nabla f^i(X_t^i)\Big\|^2 + 0.5 \big)\\
&\leq 2\frac{\eta^2}{n}B\big( \E\|\mu_{t}-x^*\|^2 + \frac{4\eta^2}{n^2}\E\Big[\frac{2L^2}{n}\sum_i \big\|X_t^i - \mu_t\big\|^2 + 4L\big[f(\mu_t) - f(x^*) \big]\Big] + 0.5 \big)\\
&\leq \eta^2 \frac{2B}{n}\E\|\mu_{t}-x^*\|^2 + \eta^4 \frac{16BL^2}{n^3}\E[\Gamma_t] + \eta^4 \frac{32BL}{n^3}\E\big[f(\mu_t) - f(x^*) \big] + \eta^2 \frac{B}{n}.
\end{split}
\end{align}
To get the (*) inequality, we first used Young's inequality twice with $\alpha = \frac{1}{2}$, to get that 
\begin{align*}
\E\|\mu_{t} - x^* \| + \E\|F'_t\| \le (\E\|\mu_{t} - x^* \|)^2 + \frac{1}{4} + (\E\|F'_t\|)^2+\frac{1}{4}
\end{align*}
and then applied Jensen's inequality to get 
\begin{align*}
(\E\|\mu_{t} - x^* \|)^2 + \frac{1}{4} + (\E\|F'_t\|)^2+\frac{1}{4}
\le \E\|\mu_{t} - x^* \|^2 +  \frac{1}{4} + \E\|F'_t\|^2+\frac{1}{4}.
\end{align*}


Then following by (\ref{eq:h first_term}), (\ref{eq:h second_term}) and (\ref{eq:h third_term}), the latter inequality (\ref{eq:h NewDecNorm}) would be: 
\begin{align*}
\E \Big\| \mu_{t+1}  -x^* \Big \|^2 &=\E \Big\|\mu_t-x^*-F'_t \Big\|^2 + \E_{X^1_t,X^2_t,...,X^n_t} \Big[ \E_t \|F'_t-F_t \|^2\Big]\\
&+ 2 \E_{X^1_t,X^2_t,...,X^n_t} \Big[ \E_t \Big \langle \mu_t-x^*-  F'_t, F'_t-F_t \Big \rangle\Big]\\
\le &(1-\frac{\ell\eta}{n})\E \|\mu_t-x^*\|^2 - (4\frac{\eta}{n}-16L\frac{\eta^2}{n^2})\E \big[f(\mu_t)-f(x^*)\big] + \big(2\frac{L+\ell}{n}\eta + 8\frac{L^2}{n^2}\eta^2\big)\E[\Gamma_t]\\
&\quad+\frac{\eta^2L^2(96d+448)}{n^2} \E[\Gamma_t] + \frac{\eta^2((96d+448)\varsigma_0^2 n_0+88\varsigma_1^2 n_1)}{n^3} \\&\quad+ \frac{48\eta^2(4d+17)L\E[f(\mu_t) - f(x^*)]}{n^2} +\frac{8\eta^2((d+4)\sum_{i \in N_0} s_i^2+\sum_{i \in N_1} s_i^2)}{n^3} + \frac{12\eta^4 n_0 L^2 (d+6)^3}{n^3c^2}.
\\
&\quad+ \eta^2 \frac{4B}{n}\E\|\mu_{t}-x^*\|^2 + \eta^4 \frac{32BL^2}{n^3}\E[\Gamma_t] + \eta^4 \frac{64BL}{n^3}\E\big[f(\mu_t) - f(x^*) \big] + \eta^2 \frac{2B}{n}.
\end{align*}
Hence:
\begin{align*}
\E &\Big\| \mu_{t+1}  -x^* \Big \|^2 \le (1-\frac{\ell\eta}{n} + \eta^2 \frac{4B}{n})\E\|\mu_t-x^*\|^2 - (4\frac{\eta}{n} - \eta^2\frac{16L(12d+52)}{n^2} - \eta^4 \frac{64BL}{n^3})\E\big[f(\mu_t)-f(x^*)\big]\\
&+ \left( 2\frac{L+\ell}{n}\eta + \eta^2L^2 \frac{96d+456}{n^2} + \eta^4 \frac{32BL^2}{n^3}\right)\E[\Gamma_t]\\
&+ \frac{\eta^2((96d+448)\varsigma_0^2 n_0+88\varsigma_1^2 n_1)}{n^3}+\frac{8\eta^2((d+4)\sum_{i \in N_0} s_i^2+\sum_{i \in N_1} s_i^2)}{n^3}+\frac{12\eta^4 n_0 L^2 (d+6)^3}{n^3c^2}+\eta^2 \frac{2B}{n}.
\end{align*}
We get the proof of the Lemma by plugging $\eta \le \frac{\sqrt{\ell cn}}{2\sqrt{Ln_0}(d+3)^\frac{3}{4}}$
and $B \le \frac{\eta n_0}{2cn}L(d+3)^{\frac{3}{2}}$ (Lemma \ref{lem:h average biasedness upper bound}) in the above inequality.
\end{proof}


\begin{theoremsub}[Convex case of Theorem \ref{thm:main}]
\label{thm:main-convex}
Assume that the functions $f$ and $f_i$ satisfy assumptions \ref{asmp:strongly_convex}, \ref{asmp:lipschitz}, \ref{asmp:global_variance} and \ref{asmp:unbiasedness_bounded_local_variance_of_F}. Let $T$ to be large enough such that $\frac{T}{\log T } = \Omega\left(\frac{n(d+n)(L+1)\left(\frac{1}{\ell}+1\right)}{\ell}\right)$, and let the learning rate be $\eta = \frac{4n\log T }{T \ell}$. For $1 \le t \le T$, let the sequence of weights $w_t$ be given by $w_t = \left(1-\frac{\eta \ell}{2n}\right)^{-t}$ and let $S_T = \sum_{t=1}^{T} w_T$. Finally, define $\mu_t = \sum_{i = 1}^n X^i_t/n$ and $y_T=\sum_{t=1}^T \frac{w_t \mu_{t-1}}{S_T}$ to be the mean over local model parameters. Then, we can show that HDO provides the following convergence rate:
\begin{align*}
\E[&f(y_T) - f(x^*)]+\frac{\ell \E\|\mu_{T}-x^*\|^2}{8} \\&=
O\Bigg(\frac{L \|\mu_0-x^*\|^2}{T\log T } +\frac{\log (T) (d n_0\varsigma_0^2+n_1 \varsigma_1^2)}{T \ell n} \\&\quad\quad\quad+
\frac{\log(T)(d n_0\sigma_0^2+n_1 \sigma_1^2)}{T \ell n} +\frac{\log(T) d n_0}{T \ell n}\Bigg).
\end{align*}
\end{theoremsub}

\begin{proof}
Let $a_t^2:=\E\|\mu_t-x^*\|^2$, $e_t:=\E\big[f(\mu_t)-f(x^*)\big]$ and
\begin{align*}
C_1&:=4\frac{\eta}{n} - \eta^2\frac{16L(12d+52)}{n^2} - \eta^4 \frac{64BL}{n^3} \ge 4\frac{\eta}{n} - \eta^2\frac{16L(12d+52)}{n^2} - \eta^5 \frac{32n_0L^2(d+3)^{\frac{3}{2}}}{d^{\frac{1}{2}}n^4},\\ 
C_2&:= 2\frac{L+\ell}{n}\eta + \eta^2L^2 \frac{96d+456}{n^2} + \eta^4 \frac{32BL^2}{n^3} \le 
2\frac{L+\ell}{n}\eta + \eta^2L^2 \frac{96d+456}{n^2} + \eta^5 \frac{16n_0L^3(d+3)^{\frac{3}{2}}}{d^{\frac{1}{2}} n^4} \\
C_3&:=\frac{\eta^2((96d+448)\varsigma_0^2 n_0+88\varsigma_1^2 n_1)}{n^3}+\frac{8\eta^2((d+4)\sum_{i \in N_0} s_i^2+\sum_{i \in N_1} s_i^2)}{n^3}+\frac{12\eta^4 n_0 L^2 (d+6)^3}{n^3c^2}+\frac{2B \eta^2}{n} \\ &\le \frac{\eta^2((96d+448)\varsigma_0^2 n_0+88\varsigma_1^2 n_1)}{n^3}+\frac{8\eta^2((d+4)\sum_{i \in N_0} s_i^2+\sum_{i \in N_1} s_i^2)}{n^3}+\frac{12\eta^4 n_0 L^2 (d+6)^3}{n^3d}+\frac{L (d+3)^{\frac{3}{2}}n_0 \eta^3}{d^{\frac{1}{2}} n^2}.
\end{align*}
Where, in the above inequalities, we used Lemma \ref{lem:h average biasedness upper bound}.
Therefore, the recursion from Lemma \ref{lem:h supmartingale} can be written as 
\begin{equation*}
    a_{t}^2 \leq (1-\frac{\ell \eta}{2n}) a_{t-1}^2 - C_1 e_{t-1} + C_2 \E[\Gamma_{t-1}] + C_3.
\end{equation*}
Then, we multiply the above recursion by $w_t = (1-\frac{\eta \ell}{2n})^{-t}$ and 
\begin{equation*}
    w_t a_{t}^2 \leq w_t\Big((1-\frac{\ell \eta}{2n}) a_{t-1}^2 - C_1 e_{t-1} + C_2 \E[\Gamma_{t-1}] + C_3\Big) =
    w_{t-1} a_{t-1}^2-w_t C_1 e_{t-1} + w_t C_2 \E[\Gamma_{t-1}] + w_t C_3.
\end{equation*}
By summing the above inequality for $t \in \{1, 2, ..., T\}$ and cancelling and rearrange terms we get:
\begin{align*}
    w_T a_T^2 \leq w_0 a_0^2 - C_1 \sum_{t=1}^T w_t e_{t-1} + C_2 \sum_{t=1}^T w_t \E [\Gamma_{t-1}] + C_3 \sum_{t=1}^T w_t.
\end{align*}
By using Lemma \ref{lem:h sum_gamma} in the inequality above we get that
\begin{align}\label{eqn:h semi-final}
\begin{split}
w_T &a_T^2 \leq w_0 a_0^2 \begin{aligned}[t]&-C_1 \sum_{t=1}^T w_t e_{t-1} + 120 C_2 \eta^2(2d+9)L \sum_{t=1}^{T-1} w_{t} \E[f(\mu_{t-1})-f(x^*)]  \\ 
&\hspace{-1mm}+ C_2 \Big(\frac{60\eta^2(2(d+4)n_0\varsigma_0^2+n_1 \varsigma_1^2)}{n}+\frac{20\eta^2(2(d+4) \sum_{i \in N_0} s_i^2+\sum_{i \in N_1} s_i^2)}{n} + \frac{10\eta^4 n_0 L^2(d+6)^3}{nc^2} \Big) \sum_{t=1}^{T-1} w_{t} \\ &\hspace{-1mm} + C_3 \sum_{t=1}^T w_t\end{aligned}\\
\leq& w_0 a_0^2 \begin{aligned}[t]&-\underbrace{(C_1 - 120 C_2 \eta^2(2d+9)L)}_{D_1:=} \sum_{t=1}^T w_t e_{t-1}\\
&\hspace{-6mm} + \underbrace{\bigg(C_2 \Big(\frac{60\eta^2(2(d+4)n_0\varsigma_0^2+n_1 \varsigma_1^2)}{n}+\frac{20\eta^2(2(d+4) \sum_{i \in N_0} s_i^2+\sum_{i \in N_1} s_i^2)}{n} + \frac{10\eta^4 n_0 L^2(d+6)^3}{nc^2} \Big) + C_3\bigg)}_{D_2:=} \sum_{t=1}^T w_t. \end{aligned}
\end{split}
\end{align}
Let $S_T:=\sum_{t=1}^T w_t$ and $y_T:= \sum_{t=1}^T \frac{w_t \mu_{t-1}}{S_T}$, then by convexity of $f$ we have that
\begin{equation}\label{eqn:h y_T}
    \E[f(y_T) - f(x^*)] \leq \frac{1}{S_T}\sum_{t=1}^T w_t e_{t-1}.
\end{equation}
By choosing small enough $\eta$, so that have $D_1 \ge 0$, we can combine  inequalities (\ref{eqn:h semi-final}) and (\ref{eqn:h y_T}) to get
\begin{align}
\begin{split}
\E[f(y_T) - f(x^*)]+\frac{w_T a_T^2}{S_T D_1} \leq \frac{w_0 a_0^2}{S_T D_1} + \frac{D_2}{D_1}
\end{split}
\end{align}
Our first goal is to lower bound $D_1$, we aim to choose upper bound on $\eta$ so that $D_1=\Omega(\frac{\eta}{n})$.
For this it will be enough to set $\eta = O(\frac{1}{d (L+\ell+1)})$.
Since $C_1 \ge 4\frac{\eta}{n} - \eta^2\frac{16L(12d+52)}{n^2} - \eta^4 \frac{32n_0L^2(d+3)^{\frac{3}{2}}}{cn^4}$,
we have $C_1 = \Omega(\frac{\eta}{n})$.
Plus, $C_2=O(1/n)$ and hence $120 C_2 \eta^2(2d+9)L = O(\frac{\eta}{n})$, thus $D_1=\Omega(\frac{\eta}{n})$, as desired.
Also:
\begin{align*}
S_t = \sum_{t=1}^T w_t = \sum_{t=1}^T (1-\frac{\eta \ell}{2n})^{-t} \ge (1-\frac{\eta \ell}{2n})^{-T} \ge e^{\frac{\eta \ell T}{2n}}.
\end{align*}
By setting $\eta = \frac{4n\log(T)}{T \ell}$, we get $S_t \geq T^2$.
Therefore, we have 
\begin{equation} \label{eqn: w0 a0^2 term}
    \frac{w_0 a_0^2}{S_T D_1} = O \left( \frac{w_0 a_0^2 \ell}{T\log(T)} \right)=
    O \left( \frac{L \|\mu_0-x^*\|^2}{T\log(T)} \right).
\end{equation}
Next, we upper bound $D_2$. Since $\eta=O(\frac{1}{d(L+\ell+1)})$ we get that   $C_2=O((L+\ell)\eta /n)$. Additionally, 
\begin{align*}
\frac{60\eta^2(2(d+4)n_0\varsigma_0^2+n_1 \varsigma_1^2)}{n}&+\frac{20\eta^2(2(d+4) \sum_{i \in N_0} s_i^2+\sum_{i \in N_1} s_i^2)}{n} + \frac{10\eta^4 n_0 L^2(d+6)^3}{n d} \\&= O\Bigg(\frac{\eta^2(d n_0\varsigma_0^2+n_1 \varsigma_1^2)}{n}+
\frac{\eta^2(d n_0\sigma_0^2+n_1 \sigma_1^2)}{n} + \frac{\eta^2 n_0}{n}\Bigg).
\end{align*}
By using  $\eta = O(\frac{1}{(L+\ell+1)n})$ in the equation above, we get $C_2 = O(\frac{1}{n^2})$ and hence
\begin{align*}
C_2\Bigg(\frac{60\eta^2(2(d+4)n_0\varsigma_0^2+n_1 \varsigma_1^2)}{n}&+\frac{20\eta^2(2(d+4) \sum_{i \in N_0} s_i^2+\sum_{i \in N_1} s_i^2)}{n} + \frac{10\eta^4 n_0 L^2(d+6)^3}{n d}\Bigg) \\&= O\Bigg(\frac{\eta^2(d n_0\varsigma_0^2+n_1 \varsigma_1^2)}{n^3}+
\frac{\eta^2(d n_0\sigma_0^2+n_1 \sigma_1^2)}{n^3} + \frac{\eta^2 n_0}{n^3}\Bigg).
\end{align*}
Finally, we have that 
\begin{align*}
C_3 = O\Bigg(\frac{\eta^2(d n_0\varsigma_0^2+n_1 \varsigma_1^2)}{n^3}+
\frac{\eta^2(d n_0\sigma_0^2+n_1 \sigma_1^2)}{n^3} + \frac{\eta^2 n_0}{n^3}+\frac{\eta^3 d L n_0}{n^2}\Bigg).
\end{align*}
If we put together the above inequalities we get that 
\begin{align*}
D_2&=O\Bigg(\frac{\eta^2(d n_0\varsigma_0^2+n_1 \varsigma_1^2)}{n^3}+
\frac{\eta^2(d n_0\sigma_0^2+n_1 \sigma_1^2)}{n^3} +\frac{\eta^3 L d n_0}{n^2}\Bigg).
\end{align*}
By plugging $D_1=\Omega(\frac{\eta}{n})$ and $\eta = \frac{4n\log(T)}{T \ell}$, in the equation above we get that
\begin{align*}
    \frac{D_2}{D_1} &= O\Bigg(\frac{\eta(d n_0\varsigma_0^2+n_1 \varsigma_1^2)}{n^2}+
\frac{\eta(d n_0\sigma_0^2+n_1 \sigma_1^2)}{n^2} +\frac{\eta^2 L d n_0}{n}\Bigg)
\\ &= O\Bigg(\frac{\log(T)(d n_0\varsigma_0^2+n_1 \varsigma_1^2)}{T \ell n}+
\frac{\log(T)(d n_0\sigma_0^2+n_1 \sigma_1^2)}{T \ell n} +\frac{\log(T) d n_0}{T \ell n}\Bigg).
\end{align*}
We also have that $D_1 \le C_1 \le \frac{4\eta}{n}$,
and 
\begin{align*}
\frac{w_T}{S_T}=\frac{(1-\frac{\eta \ell}{2n})^{-T}}{\sum_{t=1}^T (1-\frac{\eta \ell}{2n})^{-t}} \ge (1-\frac{\eta \ell}{2n}) \Big (1-\frac{\eta \ell}{2n})^{-1}-1 \Big )=
\frac{\eta \ell}{2n}.
\end{align*}
Hence, $\frac{w_T}{S_T D_1} \ge \frac{\ell}{8}.$ 
By putting together the above above inequalities we get the final convergence bound:
\begin{align*}
\E[f(y_T) - f(x^*)]&+\frac{\ell \E\|\mu_{T}-x^*\|^2}{8}=
\E[f(y_T) - f(x^*)]+\frac{\ell a_T^2}{8} \le \E[f(y_T) - f(x^*)]+\frac{w_T a_T^2}{S_T D_1} \leq \frac{w_0 a_0^2}{S_T D_1} + \frac{D_2}{D_1} \\&=O\Bigg(\frac{L \|\mu_0-x^*\|^2}{T\log(T)}) +\frac{\log(T)(d n_0\varsigma_0^2+n_1 \varsigma_1^2)}{T \ell n}+
\frac{\log(T)(d n_0\sigma_0^2+n_1 \sigma_1^2)}{T \ell n} +\frac{\log(T) d n_0}{T \ell n}\Bigg).
\end{align*}
Finally, we need to gather all upper bounds on $\eta$ and compute lower bound on $T$ so that the upper bounds are satisfied.
We need $\eta=O(\frac{1}{d (L+\ell+1)}), \eta=O(\frac{1}{n (L+\ell+1)})$
in the proof of this theorem, the lemmas we used require $\eta \le \frac{1}{10\ell}$ and 
\begin{align*}
\eta \le \frac{\sqrt{\ell c n}}{2\sqrt{Ln_0}(d+3)^\frac{3}{4}}=O\Big(\frac{\ell}{Ld}\Big).
\end{align*}

Considering $\ell \le L$, we get that we need $\eta=O\Big(\frac{1}{(d+n)(L+1)(\frac{1}{\ell}+1)}\Big)$.
Thus, we need $\frac{T}{\log(T)} = \Omega(\frac{n(d+n)(L+1)(\frac{1}{\ell}+1)}{\ell}).$
\end{proof}

\subsection{Proof of Non-Convex Case of Theorem \ref{thm:main}}

\begin{lemma}\label{lem: non-convex avg movement upper bound}
For any time step $t$, we have
\begin{equation*}
    \E \| \mu_{t+1} - \mu_t\|^2 \leq \frac{4\eta^2}{n^2}\E[M_t^G].
\end{equation*}
\end{lemma}
\begin{proof}
\begin{align*}
    \E \| \mu_{t+1} - \mu_t\|^2 &= \frac{1}{n(n-1)} \sum_i \sum_{j \neq i} \E \| \frac{\eta}{n}(G^i(X_t^i) + G^j(X_t^j))\|^2\\
    &\overset{Cauchy-Schwarz}{\leq} \frac{2\eta^2}{n^3(n-1)}\sum_i \sum_{j \neq i} \E (\|G^i(X_t^i)\|^2 + \|G^j(X_t^j)\|^2)\\ 
    &=\frac{4\eta^2}{n^3}\sum_i \E \|G^i(X_t^i)\|^2=\frac{4\eta^2}{n^2}\E[M_t^G]
\end{align*}
\end{proof}

\begin{lemma}\label{lem: non-convex M^f}
For any time step $t$ and constants $\alpha_0 \ge \alpha_1 > 0$ let $M_t^f(\alpha_0, \alpha_1)=\frac{\alpha_0}{n}\sum_{i  \in N_0} \big\|\nabla f^i(X_t^i)\big\|^2 + \frac{\alpha_1}{n}\sum_{i \in N_1} \big\|\nabla f^i(X_t^i)\big\|^2$. We have that:
\begin{equation*}
 \E[M_t^f(\alpha_0, \alpha_1)] \le 3L^2 \alpha_0 \E[\Gamma_t] + \frac{3\alpha_0 n_0\varsigma_0^2+3\alpha_1 n_1\varsigma_1^2}{n}+ 3\frac{ \alpha_0 n_0 + \alpha_1 n_1}{n} \E \| \nabla f(\mu_t)\|^2.
\end{equation*}
\end{lemma}
\begin{proof}
\begin{align}
\begin{split}
\frac{\alpha_0}{n}\sum_{i \in N_0} \E \big\|\nabla f^i(X_t^i)\big\|^2 
&= \frac{\alpha_0}{n}\sum_{i \in N_0} \E \big\|\nabla f^i(X_t^i) - \nabla f^i(\mu_t) + \nabla f^i(\mu_t)-\nabla f(\mu_t)+\nabla f(\mu_t)\big\|^2\\
&\overset{\text{Assumption}\ref{asmp:global_variance}, \text{Cauchy-Schwarz}}{\leq} \frac{3L^2 \alpha_0}{n} \sum_{i \in N_0}
\E \|X_t^i- \mu_t\|^2 + \frac{3\alpha_0 n_0\varsigma_0^2}{n}+3 \frac{\alpha_0 n_0}{n}\E\| \nabla f(\mu_t)\|^2.
\end{split}
\end{align}
Similarly, in the case of first-order nodes we get: 
\begin{align}
\begin{split}
\frac{\alpha_1}{n}\sum_{i \in N_1} \E \big\|\nabla f^i(X_t^i)\big\|^2 \le 
\frac{3L^2 \alpha_1}{n} \sum_{i \in N_1}
\E \|X_t^i- \mu_t\|^2 + \frac{3\alpha_1 n_1\varsigma_1^2}{n}+3 \frac{\alpha_1 n_1}{n}\E\| \nabla f(\mu_t)\|^2.
\end{split}
\end{align}
By summing up the above inequalities and using the fact that $\alpha_0 \ge \alpha_1$ (together with the definition of $\Gamma_t$), we get the proof of the lemma.
\end{proof}

\begin{lemma}\label{lem: non-convex M^G}
Assume $\nu: =\frac{\eta}{c}$ is fixed, where $\eta$ and $c$ are the learning rate and a constant respectively. Then, for any time step t we have:  
\begin{align*}
\E\big[ M_t^G\big] \leq 6(d+4)L^2 \E[\Gamma_t] &+ \frac{6(d+4)n_0\varsigma_0^2+3n_1 \varsigma_1^2}{n}+\frac{6(d+4)n_0+3n_1}{n}\E \| \nabla f(\mu_t)\|^2 \\&+ \frac{2(d+4) \sum_{i \in N_0} s_i^2+\sum_{i \in N_1} s_i^2}{n} + \eta^2 \frac{n_0}{2nc^2}L^2(d+6)^3.
\end{align*}
\end{lemma}
\begin{proof}
\begin{align}
\begin{split}
\E_t\big[ M_t^G\big] &= \frac{1}{n}\sum_i \E_t \big\|G^i(X_t^i)\big\|^2 \overset{ (\ref{eqn:Gv_second_moment_upper_bound})}{\leq} \frac{1}{n}\sum_{i \in N_0} (\tfrac{1}{2} \nu^2 L^2 (d+6)^3+ 2 (d+4)  \left[\E_t\|\nabla f^i(X_t^i)\|^2+s_i^2\right]) \\&\quad\quad\quad\quad\quad\quad\quad\quad\quad\quad\quad\quad\quad\quad\quad+ \frac{1}{n} \sum_{i \in N_1}  (\E_t\|\nabla f^i(X_t^i)\|^2+s_i^2)\\
&= \E_t[M_t^f(2(d+4),1)] + \frac{2(d+4) \sum_{i \in N_0} s_i^2+\sum_{i \in N_1} s_i^2}{n} + \eta^2 \frac{n_0}{2nc^2}L^2(d+6)^3.
\end{split}
\end{align}
Next, we take expectation with respect to $X_t^1, X_t^2, ..., X_t^n$ and use Lemma \ref{lem: non-convex M^f} to get
\begin{align}
\begin{split}
\E\big[ M_t^G\big] \leq 6(d+4)L^2 \E[\Gamma_t] &+ \frac{6(d+4)n_0\varsigma_0^2+3n_1 \varsigma_1^2}{n}+\frac{6(d+4)n_0+3n_1}{n}\E \| \nabla f(\mu_t)\|^2 \\&+ \frac{2(d+4) \sum_{i \in N_0} s_i^2+\sum_{i \in N_1} s_i^2}{n} + \eta^2 \frac{n_0}{2nc^2}L^2(d+6)^3,
\end{split}
\end{align}
which finishes the proof of the lemma.
\end{proof}

\begin{lemma} \label{lem: non-convex Gamma_t}
For any time step $t$ and fixed learning rate $\eta \le \frac{1}{10 L (d+4)^\frac{1}{2}}$ 
\begin{align*}
\E\big[\Gamma_{t+1}\big] \leq \big( 1 - \frac{1}{4n})\E\big[ \Gamma_t \big] &+ \frac{12\eta^2(2(d+4)n_0\varsigma_0^2+n_1 \varsigma_1^2)}{n^2}+\frac{4\eta^2(6(d+4)n_0+3n_1)}{n^2}\E \| \nabla f(\mu_t)\|^2 \\&+ \frac{4\eta^2(2(d+4) \sum_{i \in N_0} s_i^2+\sum_{i \in N_1} s_i^2)}{n^2} + \frac{2\eta^4 n_0 L^2(d+6)^3}{n^2c^2}.
\end{align*}
\end{lemma}
\begin{proof}
Note that Lemma \ref{lem:GammaBoundPerStepHelper} does not assume anything about convexity. Therefore, we have:
\begin{equation*}
    \E\big[ \Gamma_{t+1} \big] \leq \big( 1 - \frac{1}{2n}\big)\E\big[ \Gamma_t \big] + \frac{4}{n}\eta^2\E\big[M_t^G\big].
\end{equation*}
Now, by using Lemma \ref{lem: non-convex M^G} in the inequality above we have
\begin{align*}
    &\E\big[ \Gamma_{t+1} \big] \leq \big( 1 - \frac{1}{2n}\big)\E\big[ \Gamma_t \big] + \frac{4}{n}\eta^2\E\big[M_t^G\big]\\ 
    &\le \big( 1 - \frac{1}{2n})\E\big[ \Gamma_t \big] + \frac{4\eta^2}{n}\Bigg(6(d+4)L^2 \E[\Gamma_t] + \frac{6(d+4)n_0\varsigma_0^2+3n_1 \varsigma_1^2}{n}+\frac{6(d+4)n_0+3n_1}{n}\E \| \nabla f(\mu_t)\|^2 \\
    &\quad\quad\quad\quad\quad\quad\quad\quad\quad\quad\quad\quad+ \frac{2(d+4) \sum_{i \in N_0} s_i^2+\sum_{i \in N_1} s_i^2}{n} + \eta^2 \frac{n_0}{2nc^2}L^2(d+6)^3\Bigg) \\
    &=\big( 1 - \frac{1}{2n}+\frac{24\eta^2 L^2 (d+4)}{n})\E\big[ \Gamma_t \big] + \frac{12\eta^2(2(d+4)n_0\varsigma_0^2+n_1 \varsigma_1^2)}{n^2}+\frac{4\eta^2(6(d+4)n_0+3n_1)}{n^2}\E \| \nabla f(\mu_t)\|^2 \\
    &\quad\quad\quad\quad\quad\quad\quad\quad\quad\quad\quad\quad+ \frac{4\eta^2(2(d+4) \sum_{i \in N_0} s_i^2+\sum_{i \in N_1} s_i^2)}{n^2} + \frac{2\eta^4 n_0 L^2(d+6)^3}{n^2c^2}.
\end{align*}
We get the proof of the lemma by using $\eta \le \frac{1}{10 L (d+4)^\frac{1}{2}}$ in the above inequality.
\end{proof}

\begin{lemma} \label{lem: non-convex sum of Gammas}
\begin{align*}
    \sum_{t=0}^{T-1} \E[\Gamma_t] &\leq T\bigg( \frac{48\eta^2(2(d+4)n_0\varsigma_0^2+n_1 \varsigma_1^2)}{n} +\frac{16\eta^2(2(d+4) \sum_{i \in N_0} s_i^2+\sum_{i \in N_1} s_i^2)}{n} + \frac{8\eta^4 n_0 L^2(d+6)^3}{nc^2} \bigg)\\
    &\quad\quad\quad\quad+\frac{16\eta^2(6(d+4)n_0+3n_1)}{n}\sum_{t=0}^{T-2}\E \|\nabla f(\mu_t)\|^2.    
\end{align*}
\end{lemma}
\begin{proof}
Using \ref{lem: non-convex Gamma_t} we can write:
\begin{align} \label{eq: non-convex Gamma upper bound}
\begin{split}
    \E[\Gamma_{t}] &\leq \bigg( 1 + (1-\frac{1}{4n}) + \dots + (1-\frac{1}{4n})^{t-1} \bigg) . \bigg(\begin{aligned}[t]& \frac{12\eta^2(2(d+4)n_0\varsigma_0^2+n_1 \varsigma_1^2)}{n^2} \\ 
    &+\frac{4\eta^2(2(d+4) \sum_{i \in N_0} s_i^2+\sum_{i \in N_1} s_i^2)}{n^2} + \frac{2\eta^4 n_0 L^2(d+6)^3}{n^2c^2} \bigg)\end{aligned}\\
    &\quad\quad\quad\quad+\frac{4\eta^2(6(d+4)n_0+3n_1)}{n^2}\sum_{i=0}^{t-1} (1-\frac{1}{4n})^{t-1-i}\E \|\nabla f(\mu_i)\|^2 + (1-\frac{1}{4n})^{t}\underbrace{\E[\Gamma_0]}_{0}\\
    &\leq \bigg( \underbrace{\sum_{i=0}^\infty (1-\frac{1}{4n})^i}_{4n} \bigg).\bigg(\begin{aligned}[t]& \frac{12\eta^2(2(d+4)n_0\varsigma_0^2+n_1 \varsigma_1^2)}{n^2} \\ 
    &+\frac{4\eta^2(2(d+4) \sum_{i \in N_0} s_i^2+\sum_{i \in N_1} s_i^2)}{n^2} + \frac{2\eta^4 n_0 L^2(d+6)^3}{n^2c^2} \bigg)\end{aligned}\\
    &\quad\quad\quad\quad+\frac{4\eta^2(6(d+4)n_0+3n_1)}{n^2}\sum_{i=0}^{t-1} (1-\frac{1}{4n})^{t-1-i}\E \|\nabla f(\mu_i)\|^2\\
    &\leq \frac{48\eta^2(2(d+4)n_0\varsigma_0^2+n_1 \varsigma_1^2)}{n} +\frac{16\eta^2(2(d+4) \sum_{i \in N_0} s_i^2+\sum_{i \in N_1} s_i^2)}{n} + \frac{8\eta^4 n_0 L^2(d+6)^3}{nc^2}\\
    &\quad\quad\quad\quad+\frac{4\eta^2(6(d+4)n_0+3n_1)}{n^2}\sum_{i=0}^{t-1} (1-\frac{1}{4n})^{t-1-i}\E \|\nabla f(\mu_i)\|^2.
\end{split}
\end{align}

Now, we sum up the inequality \ref{eq: non-convex Gamma upper bound} for $t \in \{ 1, ..., T-1\}$ and we have:

\begin{align}
\begin{split}
    \sum_{t=1}^{T-1}\E[\Gamma_t] &\leq T\bigg( \frac{48\eta^2(2(d+4)n_0\varsigma_0^2+n_1 \varsigma_1^2)}{n} +\frac{16\eta^2(2(d+4) \sum_{i \in N_0} s_i^2+\sum_{i \in N_1} s_i^2)}{n} + \frac{8\eta^4 n_0 L^2(d+6)^3}{nc^2} \bigg)\\
    &\quad\quad\quad\quad+\frac{4\eta^2(6(d+4)n_0+3n_1)}{n^2}\sum_{t=1}^{T-1}\sum_{i=0}^{t-1} (1-\frac{1}{4n})^{t-1-i}\E \|\nabla f(\mu_i)\|^2\\
    &= T\bigg( \frac{48\eta^2(2(d+4)n_0\varsigma_0^2+n_1 \varsigma_1^2)}{n} +\frac{16\eta^2(2(d+4) \sum_{i \in N_0} s_i^2+\sum_{i \in N_1} s_i^2)}{n} + \frac{8\eta^4 n_0 L^2(d+6)^3}{nc^2} \bigg)\\
    &\quad\quad\quad\quad+\frac{4\eta^2(6(d+4)n_0+3n_1)}{n^2}\sum_{t=0}^{T-2}\E \|\nabla f(\mu_t)\|^2\underbrace{\sum_{i=0}^{T-2-t} (1-\frac{1}{4n})^{i}}_{\leq 4n}\\
    &\leq T\bigg( \frac{48\eta^2(2(d+4)n_0\varsigma_0^2+n_1 \varsigma_1^2)}{n} +\frac{16\eta^2(2(d+4) \sum_{i \in N_0} s_i^2+\sum_{i \in N_1} s_i^2)}{n} + \frac{8\eta^4 n_0 L^2(d+6)^3}{nc^2} \bigg)\\
    &\quad\quad\quad\quad+\frac{16\eta^2(6(d+4)n_0+3n_1)}{n}\sum_{t=0}^{T-2}\E \|\nabla f(\mu_t)\|^2.
\end{split}
\end{align}
By adding $\Gamma_0 =0$ to the left hand side, the proof of the lemma gets finished.
\end{proof}

\begin{lemma}\label{lem: non-convex dot upper bound}
For any time step $t$,
\begin{equation*}
    \E \langle \nabla f(\mu_t), \mu_{t+1} - \mu_{t} \rangle \leq -\frac{\eta}{n} \E\| \nabla f(\mu_t) \|^2 + \frac{2\eta}{n}B^2 + \frac{2L^2\eta}{n} \E[\Gamma_t].
\end{equation*}
\end{lemma}
\begin{proof}
\begin{align}\label{eq: non-convex dot decomposition}
    \nonumber \E_t\langle \nabla f(\mu_t), \mu_{t+1} - \mu_{t} \rangle &= \langle \nabla f(\mu_t), \E_t[\mu_{t+1}] - \mu_{t} \rangle = \langle \nabla f(\mu_t), -\frac{\eta}{n^2(n-1)}\sum_{i}\sum_{j\neq i} \big( \E_t[G^i(X_t^i)] + \E_t[G^i(X_t^j)] \big) \rangle\\
    \nonumber &= \langle \nabla f(\mu_t), -\frac{2\eta}{n^2}\sum_{i} \E_t[G^i(X_t^i)] \rangle = \frac{2\eta}{n^2}\sum_i \langle \nabla f(\mu_t), - \E_t[G^i(X_t^i)] \rangle\\
    &= \frac{2\eta}{n^2}\sum_i \langle \nabla f(\mu_t), - \E_t[G^i(X_t^i)] + \nabla f^i(X_t^i) - \nabla f^i(X_t^i) + \nabla f^i(\mu_t) - \nabla f^i(\mu_t) \rangle
\end{align}
\begin{align}\label{eq: non-convex decompose_1}
    \nonumber \frac{1}{n} \sum_i \langle \nabla f(\mu_t), &-\E_t[G^i(X_t^i)] + \nabla f^i (X_t^i) \rangle \leq \frac{1}{n} \sum_i \| \nabla f(\mu_t)\|.\| \E_t[G^i(X_t^i)] - \nabla f^i (X_t^i) \| \\ &\overset{\ref{def: beta_i}}{\leq} \frac{1}{n}\|\nabla f(\mu_t)\| \sum_i b_i \overset{\ref{def: B}}{=}B\|\nabla f(\mu_t)\| \overset{AM-GM}{\leq} \frac{1}{4}\| \nabla f(\mu_t)\|^2 + B^2 
\end{align}
\begin{align}\label{eq: non-convex decompose_2}
    \nonumber \frac{1}{n} \sum_i &\langle \nabla f(\mu_t), -\nabla f^i (X_t^i) + \nabla f^i(\mu_t) \rangle \leq \frac{1}{n} \sum_i \| \nabla f(\mu_t)\| .\|\nabla f^i(X_t^i) - \nabla f^i(\mu_t)\| \\
    \nonumber &\overset{Cauchy-Schwarz}{\leq} \|\nabla f(\mu_t)\| \bigg( \frac{1}{n} \sum_i \|\nabla f^i(X_t^i) - \nabla f^i(\mu_t)\|^2 \bigg)^{1/2}
    \overset{\ref{asmp:lipschitz}}{\leq} \| \nabla f(\mu_t)\| \big( L^2 \Gamma_t \big)^{1/2} \\
    &\overset{AM-GM}{\leq} \frac{1}{4}\| \nabla f(\mu_t)\|^2 + L^2 \Gamma_t
\end{align}
\begin{align}\label{eq: non-convex decompose_3}
    \frac{1}{n} \sum_i \langle \nabla f(\mu_t), -\nabla f^i(\mu_t) \rangle = \langle \nabla f(\mu_t), -\frac{1}{n} \sum_i \nabla f^i(\mu_t) \rangle = -\| \nabla f(\mu_t)\|^2
\end{align}
Combining \ref{eq: non-convex decompose_1}, \ref{eq: non-convex decompose_2}, and \ref{eq: non-convex decompose_3}, we will get
\begin{align}
    \frac{1}{n}\sum_i \langle \nabla f(\mu_t), - \E_t[G^i(X_t^i)] + \nabla f^i(X_t^i) - \nabla f^i(X_t^i) + \nabla f^i(\mu_t) - \nabla f^i(\mu_t) \rangle \leq -\frac{1}{2} \| \nabla f(\mu_t) \|^2 + B^2 + L^2 \Gamma_t.
\end{align}
By plugging in the result above in \ref{eq: non-convex dot decomposition} and taking the expectation with respect to $X_t^1, X_t^2, ..., X_t^n$ from both sides, the proof would be finished.
\end{proof}

\begin{lemma} \label{lem: non-convex mu_t difference}
For every time step $t$, we have
\begin{align*}
    \E[f(\mu_{t+1}) - f(\mu_t)] &\leq \bigg( \frac{2L\big( 6(d+4)n_0+3n_1\big)\eta^2}{n^3} - \frac{\eta}{n}\bigg)\E\|\nabla f(\mu_t)\|^2 + \bigg( \frac{2L^2\eta}{n} + \frac{12L(d+4)\eta^2}{n^2}\bigg)\E[\Gamma_t] \\
    &+ \eta^2\frac{2L}{n^2}\bigg( \frac{6(d+4)n_0\varsigma_0^2+3n_1 \varsigma_1^2}{n} +\frac{2(d+4) \sum_{i \in N_0} s_i^2+\sum_{i \in N_1} s_i^2}{n}\bigg) \\
    &+\eta^3\frac{n_0^2L^2}{2 c^2 n^3}(d+3)^3 + \eta^4\frac{n_0 L^3} {c^2 n^3}(d+6)^3.
\end{align*}
\end{lemma}
\begin{proof}
Using L-smoothness \ref{asmp:lipschitz} we can write
\begin{align*}
    \E_t[f(\mu_{t+1})] \leq f(\mu_{t}) + \E_t\langle \nabla f(\mu_t), \mu_{t+1} - \mu_{t} \rangle + \frac{L}{2} \E_t \| \mu_{t+1} - \mu_t\|^2.
\end{align*}
By taking the expectation with respect to $X_t^1, X_t^2, ..., X_t^n$ we will get
\begin{equation}
    \E[f(\mu_{t+1})] - \E[f(\mu_{t})] \leq \E\langle \nabla f(\mu_t), \mu_{t+1} - \mu_{t} \rangle + \frac{L}{2} \E \| \mu_{t+1} - \mu_t\|^2.
\end{equation}
Then, we use Lemmas \ref{lem:h average biasedness upper bound}, \ref{lem: non-convex avg movement upper bound}, and \ref{lem: non-convex dot upper bound}:
\begin{align}
\begin{split}
    \nonumber \E[f(\mu_{t+1})] - \E[f(\mu_{t})] &\leq -\frac{\eta}{n} \E\| \nabla f(\mu_t) \|^2 + \frac{2\eta}{n}B^2 + \frac{2L^2\eta}{n} \E[\Gamma_t] + \frac{2L\eta^2}{n^2}\E[M_t^G]\\
    &\overset{\ref{lem: non-convex M^G}}{\leq} -\frac{\eta}{n} \E\| \nabla f(\mu_t) \|^2 + \frac{2\eta}{n}B^2 + \frac{2L^2\eta}{n} \E[\Gamma_t] \\
    & \quad\quad+ \frac{2L\eta^2}{n^2}\begin{aligned}[t]\bigg( &6(d+4)L^2 \E[\Gamma_t] + \frac{6(d+4)n_0\varsigma_0^2+3n_1 \varsigma_1^2}{n}+\frac{6(d+4)n_0+3n_1}{n}\E \| \nabla f(\mu_t)\|^2 \\ &+\frac{2(d+4) \sum_{i \in N_0} s_i^2+\sum_{i \in N_1} s_i^2}{n} + \eta^2 \frac{n_0}{2nc^2}L^2(d+6)^3 \bigg) \end{aligned} \\
    &\overset{\ref{lem:h average biasedness upper bound}}{\leq} \bigg( \frac{2L\big( 6(d+4)n_0+3n_1\big)\eta^2}{n^3} - \frac{\eta}{n}\bigg)\E\|\nabla f(\mu_t)\|^2 + \bigg( \frac{2L^2\eta}{n} + \frac{12L(d+4)\eta^2}{n^2}\bigg)\E[\Gamma_t] \\
    &+ \eta^2\frac{2L}{n^2}\bigg( \frac{6(d+4)n_0\varsigma_0^2+3n_1 \varsigma_1^2}{n} +\frac{2(d+4) \sum_{i \in N_0} s_i^2+\sum_{i \in N_1} s_i^2}{n}\bigg) \\
    &+\eta^3\frac{n_0^2L^2}{2 c^2 n^3}(d+3)^3 + \eta^4\frac{n_0 L^3} {c^2 n^3}(d+6)^3.
\end{split}
\end{align}

\end{proof}

\begin{theoremsub}[Non-Convex case of Theorem \ref{thm:main}]
\label{thm:main-nonconvex}
Under Assumptions \ref{asmp:lipschitz}, \ref{asmp:global_variance} and \ref{asmp:unbiasedness_bounded_local_variance_of_F}, and letting $T$ to be large enough such that $T = \Omega\left(\max\big\{\frac{L^2(dn_0 + n_1)^2}{dn^2}, n^2L^2, nn_0^3d\big\}\right)$ we have
\begin{align*}
    \frac{1}{T}\sum_{t=0}^{T-1}\E\|\nabla f(\mu_t)\|^2 \leq O \Bigg( (f(\mu_0) - f^*)+ L \big( \frac{dn_0\varsigma_0^2+n_1 \varsigma_1^2}{nd} \big)
    + L \big( \frac{d n_0\sigma_0^2+n_1\sigma_1^2}{nd} \big)
    + L^2\sqrt{\frac{n_0}{n}}\Bigg)\sqrt{\frac{d}{T}}.
\end{align*}
\end{theoremsub}

\begin{proof}
\begin{align*}
     \E[f(\mu_{t+1})] - \E[f(\mu_{t})] &\overset{\ref{lem: non-convex mu_t difference}}{\leq} \bigg( \frac{2L\big( 6(d+4)n_0+3n_1\big)\eta^2}{n^3} - \frac{\eta}{n}\bigg)\E\|\nabla f(\mu_t)\|^2 + \underbrace{\bigg( \frac{2L^2\eta}{n} + \frac{12L(d+4)\eta^2}{n^2}\bigg)}_{J_0:=}\E[\Gamma_t] \\
    &+ \eta^2\underbrace{\frac{2L}{n^2}\bigg( \frac{6(d+4)n_0\varsigma_0^2+3n_1 \varsigma_1^2}{n} +\frac{2(d+4) \sum_{i \in N_0} s_i^2+\sum_{i \in N_1} s_i^2}{n}\bigg)}_{J_1:=} \\
    &+\eta^3\underbrace{\frac{n_0^2L^2}{2 c^2 n^3}(d+3)^3}_{J_2:=} + \eta^4\underbrace{\frac{n_0 L^3} {c^2 n^3}(d+6)^3}_{J_3:=}
\end{align*}
Considering $\eta \leq \frac{n^2}{4L \big( 6(d+4)n_0 + 3n_1 \big)}$, we can write:
\begin{align}
    \nonumber \E[f(\mu_{t+1})] - \E[f(\mu_{t})] &\leq -\frac{\eta}{2n}\E\|\nabla f(\mu_t)\|^2 + J_0\E[\Gamma_t] + \eta^2 J_1 + \eta^3 J_2 + \eta^4 J_3.
\end{align}
Now, we sum up the inequality above for $t \in \{0, ..., T-1\}$ to get:

\begin{align}
    \nonumber \E[f(\mu_{T})] - \E[f(\mu_0)] \leq 
    & -\frac{\eta}{2n}\sum_{t=0}^{T-1}\E\|\nabla f(\mu_t)\|^2 + J_0\sum_{t=0}^{T-1}\E[\Gamma_t] \\
    & + \eta^2 T J_1 + \eta^3 T J_2 + \eta^4 T J_3.
\end{align}

By using Lemma \ref{lem: non-convex sum of Gammas} and the fact that $f(\mu_T) \geq f^*$ we will have:
\begin{align} \label{eqn: non-convex final ineq J}
    \nonumber \frac{\eta}{2n}\sum_{t=0}^{T-1}\E\|\nabla f(\mu_t)\|^2 \leq &\big( f(\mu_0) - f^* \big) + \eta^2 T J_1 + \eta^3 T J_2 + \eta^4 T J_3 \\
    \nonumber  &+J_0\Bigg( T\bigg( \frac{48\eta^2(2(d+4)n_0\varsigma_0^2+n_1 \varsigma_1^2)}{n} +\frac{16\eta^2(2(d+4) \sum_{i \in N_0} s_i^2+\sum_{i \in N_1} s_i^2)}{n}\\
    &\quad\quad\quad+ \frac{8\eta^4 n_0 L^2(d+6)^3}{nc^2} \bigg)+\frac{16\eta^2(6(d+4)n_0+3n_1)}{n}\sum_{t=0}^{T-2}\E \|\nabla f(\mu_t)\|^2 \Bigg) .
\end{align}
Then, by defining $K_1$, $K_2$, $K_3$, $K_4$, $K_5$, $K_6$, and $K_7$ as
\begin{align*}
    &K_1:= J_1,\\
    &K_2:= J_2+ \frac{2L^2}{n}\bigg( \frac{48(2(d+4)n_0\varsigma_0^2+n_1 \varsigma_1^2)}{n} +\frac{16(2(d+4) \sum_{i \in N_0} s_i^2+\sum_{i \in N_1} s_i^2)}{n}\bigg),\\
    &K_3:= J_3 + \frac{12L(d+4)}{n^2}\bigg( \frac{48(2(d+4)n_0\varsigma_0^2+n_1 \varsigma_1^2)}{n} +\frac{16(2(d+4) \sum_{i \in N_0} s_i^2+\sum_{i \in N_1} s_i^2)}{n}\bigg),\\
    &K_4:= \frac{16 n_0 L^4(d+6)^3}{n^2c^2},\\
    &K_5:= \frac{96 n_0 L^3(d+4)(d+6)^3}{n^3c^2},\\
    &K_6:= \frac{64L^2\eta^2(6(d+4)n_0+3n_1)}{n},\\
    &K_7:= \frac{384L\eta^3(d+4)(6(d+4)n_0+3n_1)}{n^2},
\end{align*}
we can rewrite \ref{eqn: non-convex final ineq J} as 
\begin{equation*}
    \frac{\eta}{2n}\big( 1 - K_6 - K_7\big)\sum_{t=0}^{T-1}\E\|\nabla f(\mu_t)\|^2 \leq \big( f(\mu_0) - f^* \big) + \eta^2 T K_1 + \eta^3 T K_2 + \eta^4 T K_3 + \eta^5 T K_4 + \eta^6 T K_5.
\end{equation*}
By setting $\eta \leq \sqrt{\frac{n}{256L^2\big( 6(d+4)n_0 + 3n_1 \big)}}$ and $\eta \leq \bigg(\frac{n^2}{1536L(d+4)\big( 6(d+4)n_0 + 3n_1 \big)}\bigg)^{\frac{1}{3}}$, we would get $1 - K_6 - K_7 \geq \frac{1}{2}$ and therefore we have:
\begin{equation*}
    \frac{1}{4nT}\sum_{t=0}^{T-1}\E\|\nabla f(\mu_t)\|^2 \leq \frac{1}{\eta T}\big( f(\mu_0) - f^* \big) + \eta K_1 + \eta^2 K_2 + \eta^3 K_3 + \eta^4 K_4 + \eta^5 K_5.
\end{equation*}

In this step, we set $\eta$ equal to $\frac{en}{\sqrt{T}}$ for a parameter $e$ that we will define later.

\begin{equation} \label{eqn: non-convex K_1, K_2, K_3, K_4}
    \frac{1}{4nT}\sum_{t=0}^{T-1}\E\|\nabla f(\mu_t)\|^2 \leq \frac{1}{en\sqrt{T}}\big( f(\mu_0) - f^* \big) + \frac{en}{\sqrt{T}} K_1 + \big(\frac{en}{\sqrt{T}}\big)^2 K_2 + \big( \frac{en}{\sqrt{T}}\big)^3 K_3 + \big(\frac{en}{\sqrt{T}}\big)^4 K_4 + \big(\frac{en}{\sqrt{T}}\big)^5 K_5.
\end{equation}

By putting $c = \sqrt{d}$, we can give upper bound for $K_1$, $K_2$, $K_3$, $K_4$, and $K_5$ as follows:  

\begin{align*}
    K_1&= \frac{2L}{n^2}\bigg( \frac{6(d+4)n_0\varsigma_0^2+3n_1 \varsigma_1^2}{n} +\frac{2(d+4) \sum_{i \in N_0} s_i^2+\sum_{i \in N_1} s_i^2}{n}\bigg)\\
    &= O \Bigg( \frac{L}{n^2}\bigg( \big( \frac{dn_0\varsigma_0^2+n_1 \varsigma_1^2}{n} \big) + \big( \frac{d \sum_{i \in N_0} s_i^2+\sum_{i \in N_1} s_i^2}{n} \big) \bigg)\Bigg),\\
    K_2&= \frac{n_0^2L^2}{2 c^2 n^3}(d+3)^3 + \frac{2L^2}{n}\bigg( \frac{48(2(d+4)n_0\varsigma_0^2+n_1 \varsigma_1^2)}{n} +\frac{16(2(d+4) \sum_{i \in N_0} s_i^2+\sum_{i \in N_1} s_i^2)}{n}\bigg)\\
    &=O \Bigg( \frac{n_0^2L^2d^2}{ n^3} + \frac{L^2}{n} \bigg( \big( \frac{dn_0\varsigma_0^2+n_1 \varsigma_1^2}{n} \big) + \big( \frac{d \sum_{i \in N_0} s_i^2+\sum_{i \in N_1} s_i^2}{n} \big) \bigg) \Bigg),\\
    K_3&= \frac{n_0 L^3} {c^2 n^3}(d+6)^3 + \frac{12L(d+4)}{n^2}\bigg( \frac{48(2(d+4)n_0\varsigma_0^2+n_1 \varsigma_1^2)}{n} +\frac{16(2(d+4) \sum_{i \in N_0} s_i^2+\sum_{i \in N_1} s_i^2)}{n}\bigg)\\
    &= O \Bigg( \frac{n_0L^3d^2}{ n^3} + \frac{Ld}{n^2} \bigg( \big( \frac{dn_0\varsigma_0^2+n_1 \varsigma_1^2}{n} \big) + \big( \frac{d \sum_{i \in N_0} s_i^2+\sum_{i \in N_1} s_i^2}{n} \big) \bigg) \Bigg),\\
    K_4&= \frac{16 n_0 L^4(d+6)^3}{n^2c^2} = O \Bigg(\frac{n_0 L^4d^2}{n^2}\Bigg),\\
    K_5&= \frac{96 n_0 L^3(d+4)(d+6)^3}{n^3c^2} = O \Bigg(\frac{n_0 L^3d^3}{n^3}\Bigg).
\end{align*}

Now we multiply both sides of \ref{eqn: non-convex K_1, K_2, K_3, K_4} by $4n$ and use the inequalities above to get:

\begin{align*}
    \frac{1}{T}\sum_{t=0}^{T-1}\E\|\nabla f(\mu_t)\|^2 \leq O \Bigg(& \frac{1}{e\sqrt{T}}(f(\mu_0) - f^*) + \frac{e^2n_0^2L^2d^2}{T} + \frac{e^3nn_0L^3d^2}{T^\frac{3}{2}} + \frac{e^4n^3n_0 L^4d^2}{T^2} + \frac{e^5n^3n_0 L^3d^3}{T^\frac{5}{2}}\\
    &+ \big(\frac{eL}{\sqrt{T}} + \frac{e^2n^2L^2}{T} + \frac{e^3 n^2Ld}{T^\frac{3}{2}}\big)\bigg( \big( \frac{dn_0\varsigma_0^2+n_1 \varsigma_1^2}{n} \big) + \big( \frac{d \sum_{i \in N_0} s_i^2+\sum_{i \in N_1} s_i^2}{n} \big) \bigg)\Bigg).
\end{align*}

Next, note that time $T$ here counts \emph{total interactions}. However, $\Theta(n)$ interactions occur simultaneously. Therefore, if we define $T_p$ as the \emph{parallel} execution time of the process, we deduce $T = \Theta(n T_p)$, hence we can write:

\begin{align*}
    \frac{1}{T}\sum_{t=0}^{T-1}\E\|\nabla f(\mu_t)\|^2 \leq O \Bigg(& \frac{1}{e\sqrt{nT_p}}(f(\mu_0) - f^*) + \frac{e^2n_0^2L^2d^2}{nT_p} + \frac{e^3n_0L^3d^2}{\sqrt{n}T_p^\frac{3}{2}} + \frac{e^4nn_0 L^4d^2}{T_p^2} + \frac{e^5\sqrt{n}n_0 L^3d^3}{T_p^\frac{5}{2}}\\
    &+ \big(\frac{eL}{\sqrt{nT_p}} + \frac{e^2nL^2}{T_p} + \frac{e^3 \sqrt{n}Ld}{T_p^\frac{3}{2}}\big)\bigg( \big( \frac{dn_0\varsigma_0^2+n_1 \varsigma_1^2}{n} \big) + \big( \frac{d \sum_{i \in N_0} s_i^2+\sum_{i \in N_1} s_i^2}{n} \big) \bigg)\Bigg).
\end{align*}



Considering $e \leq \frac{1}{\sqrt{d}}$ and $T_p \geq \max(\frac{n^3L^2}{d}, n)$, we will have

\begin{align*}
    \frac{1}{T}\sum_{t=0}^{T-1}\E\|\nabla f(\mu_t)\|^2 \leq O \Bigg(& \frac{1}{e\sqrt{nT_p}}(f(\mu_0) - f^*) + \frac{e^2n_0^2L^2d^2}{nT_p}\\
    &+ \big(\frac{eL}{\sqrt{nT_p}} \big)\bigg( \big( \frac{dn_0\varsigma_0^2+n_1 \varsigma_1^2}{n} \big) + \big( \frac{d \sum_{i \in N_0} s_i^2+\sum_{i \in N_1} s_i^2}{n} \big) \bigg)\Bigg).
\end{align*}

Finally, by putting $e = \frac{1}{\sqrt{d}}$ and $T_p \geq n_0^3d$, we will conclude the final convergence

\begin{align*}
    \frac{1}{T}\sum_{t=0}^{T-1}\E\|\nabla f(\mu_t)\|^2 \leq O \Bigg(& (f(\mu_0) - f^*)
    + L\bigg( \big( \frac{dn_0\varsigma_0^2+n_1 \varsigma_1^2}{nd} \big) + \big( \frac{d n_0\sigma_0^2+n_1\sigma_1^2}{nd} \big) \bigg)
    + L^2\sqrt{\frac{n_0}{n}}\Bigg)\sqrt{\frac{d}{T}}.
\end{align*}

Gathering all the assumed upper bounds for $\eta$, we have 

\begin{align*}
\eta &= O\left(\min\bigg\{\sqrt{\frac{n}{L^2 (d n_0 + n_1)}}, \left(\frac{n^2}{L (d n_0 + n_1)}\right)^{\frac{1}{3}}, \frac{n^2}{L (d n_0 + n_1)}, \frac{1}{\sqrt{L d}}\bigg\}\right),
\end{align*}

which implies $\eta = O\left(\min\big\{\frac{1}{L\sqrt{d}}, \frac{n^2}{L (d n_0 + n_1)}\big\}\right)$, resulting in $T = \Omega\left(\max\big\{\frac{L^2(dn_0 + n_1)^2}{dn^2}, n^2L^2, nn_0^3d\big\}\right)$.





\end{proof}

\end{document}